\documentclass[journal]{IEEEtran}

\usepackage{setspace}
\usepackage{comment}

\usepackage[english]{babel}
\usepackage[utf8x]{inputenc}
\usepackage[T1]{fontenc}

\usepackage{amsmath,amssymb,mathrsfs,bbm}
\usepackage{graphicx}
\usepackage[colorinlistoftodos]{todonotes}
\usepackage[colorlinks=true, allcolors=blue]{hyperref}
\usepackage{multirow}
\usepackage{makecell}
\usepackage[lined,linesnumbered,ruled]{algorithm2e}

\usepackage{psfrag,epsfig,graphics}
\usepackage{amsmath,amsthm,amssymb,multirow}
\usepackage{mathbbol}
\usepackage{amssymb}   
\usepackage{url}  
\usepackage{xstring}
\usepackage{cite}

\usepackage{mathabx} %

\DeclareSymbolFontAlphabet{\amsmathbb}{AMSb}%

\newcommand{\cp}[1]{\ifmmode {\mathcal{#1}}\else ${\mathcal{#1}}$\fi}
	
\newcommand{\bA}{\boldsymbol{A}}
\newcommand{\bB}{\boldsymbol{B}}
\newcommand{\bC}{\boldsymbol{C}}
\newcommand{\bD}{\boldsymbol{D}}

\newcommand{\bE}{\boldsymbol{E}}

\newcommand{\bI}{\boldsymbol{I}}
\newcommand{\bJ}{\boldsymbol{J}}

\newcommand{\bM}{\boldsymbol{M}}

\newcommand{\bP}{\boldsymbol{P}}
\newcommand{\bQ}{\boldsymbol{Q}}

\newcommand{\bS}{\boldsymbol{S}}

\newcommand{\bT}{\boldsymbol{T}}
\newcommand{\bU}{\boldsymbol{U}}
\newcommand{\bV}{\boldsymbol{V}}
\newcommand{\bW}{\boldsymbol{W}}
\newcommand{\bX}{\boldsymbol{X}}
\newcommand{\bY}{\boldsymbol{Y}}
\newcommand{\bZ}{\boldsymbol{Z}}

\newcommand{\ba}{\boldsymbol{a}}
\newcommand{\bb}{\boldsymbol{b}}
\newcommand{\bc}{\boldsymbol{c}}

\newcommand{\bu}{\boldsymbol{u}}
\newcommand{\bv}{\boldsymbol{v}}
\newcommand{\bx}{\boldsymbol{x}}
\newcommand{\bw}{\boldsymbol{w}}

\newcommand{\bDelta}{\boldsymbol{\Delta}}

\newcommand{\cb}[1]{\boldsymbol{#1}}

\newcommand\tensor[1]{%
  \ifcat\noexpand#1\relax %
    \mathbb{#1}%
  \else
      \if\relax\detokenize\expandafter{\romannumeral-0#1}\relax  %
        \mathbb{#1}
      \else
        \boldsymbol{\mathcal{#1}}%
      \fi
  \fi }

\newcommand{\unfold}[2]{{\mathop{#1}}_{\langle{#2}\rangle}}

\newcommand{\rank}{\operatorname{rank}}
\newcommand{\kr}{\operatorname{kr}}

\newcommand{\Ex}{\amsmathbb{E}}

\newcommand{\cdidx}{\mathsf{m}}

\usepackage{color}  %
\def\cred{\textcolor{red}}

\definecolor{darkgreen}{rgb}{0., 0.4, 0.}

\newtheorem{theorem}{Theorem}%
\newtheorem{corollary}{Corollary}
\newtheorem{lemma}{Lemma}
\newtheorem{definition}{Definition}
\newtheorem{remark}{Remark}

\newtheorem{example}{Example}

\usepackage{scalerel,stackengine}
\stackMath

\newcommand\mywidehat[1]{%
\savestack{\tmpbox}{\stretchto{%
  \scaleto{%
    \scalerel*[\widthof{\ensuremath{#1}}]{\kern3pt\mathchar"0362\kern3pt}%
    {\rule{0ex}{\textheight}}%
  }{\textheight}%
}{2.4ex}}%
\stackon[-7.2pt]{#1}{\tmpbox}{}%
}

\newcommand{\bPi}{\boldsymbol{\Pi}}
\newcommand{\bLambda}{\boldsymbol{\Lambda}}

\definecolor{red}{rgb}{0., 0.0, 0.0}

\usepackage{mathtools}

\title{
Personalized Coupled Tensor  \\ Decomposition for Multimodal Data  \\ Fusion: Uniqueness and Algorithms
\thanks{This work has been partially supported by the French National Research Agency under grants ANR-23-CE23-0024, ANR-23-CE94-0001, ANR-19-CE23-0021 and by the National Science Foundation, under grant NSF 2316420.}
}

\author{Ricardo A. Borsoi, Konstantin Usevich, David Brie, T\"ulay Adali
\thanks{R. Borsoi, K Usevich, and D. Brie are with the Universit{\'e} de Lorraine, CNRS, CRAN, Vandoeuvre-l{\`e}s-Nancy, France. e-mail: \mbox{\{firstname.lastname\}@univ-lorraine.fr}.}
\thanks{T. Adali is with the Department of Computer Science and Electrical Engineering, University of Maryland Baltimore County, Baltimore, MD 21250, USA. e-mail: \mbox{\{lastname\}@umbc.edu)}.} }

\begin{document}
\maketitle

\begin{abstract}
Coupled tensor decompositions (CTDs) perform data fusion by linking factors from different datasets. 
Although many CTDs have been already proposed, current works do not address important challenges of data fusion, where: 1) the datasets are often heterogeneous, constituting different ``views'' of a given phenomena (multimodality); and 2) each dataset can contain \textit{personalized} or dataset-specific information, constituting distinct factors that are not coupled with other datasets.
In this work, we introduce a personalized CTD framework tackling these challenges. A flexible model is proposed where each dataset is represented as the sum of two components, one related to a common tensor through a multilinear measurement model, and another specific to each dataset. Both the common and distinct components are assumed to admit a polyadic decomposition. This generalizes several existing CTD models. 
We provide conditions for specific and generic uniqueness of the decomposition that are easy to interpret. These conditions employ \textit{uni-mode} uniqueness of different individual datasets and properties of the measurement model. %
Two algorithms are proposed to compute the common and distinct components: a semi-algebraic one and a coordinate-descent optimization method.
Experimental results illustrate the advantage of the proposed framework compared with the state of the art approaches.

\end{abstract}


\section{Introduction}

Coupled tensor decompositions (CTDs) perform data fusion by linking factors among different datasets, benefiting from the complementary information across modalities~\cite{sorber2015structuredDataFusionReview,acar2013dataFusionThroughCoupledFactorizations}.
CTD has recently drawn increasing interest in various disciplines, particularly in data mining and multimodal data fusion~\cite{lahat2015multimodalFusionReview,sorber2015structuredDataFusionReview,acar2014structureRevealingFusion,acar2013dataFusionThroughCoupledFactorizations,chatzichristos2022coupledTensorDataFusionBook}.
Applications include social network data analysis~\cite{papalexakis2014turboSMT_matrix_tensor_fact} and link prediction~\cite{ermics2015linkPredictionCoupledTensor}, multimodal image fusion~\cite{prevost2020coupledTucker_hyperspectralSRR_TSP,kanatsoulis2018hyperspectralSRR_coupledCPD,borsoi2020tensorHSRvariability,prevost2022LL1_HSR}, multi-subject fMRI data fusion~\cite{kuang2023coupledCPD_fMRI_multigroup,borsoi2022coupledCP_dec_fMRI}, and several problems in array signal processing~\cite{sorensen2016harmonicRetrievalCCPD_part1,sorensen2016harmonicRetrievalCCPD_part2,sorensen2016ESPRIT_coupledCPD,sorensen2018coupledCPD_arrayProcessing,zheng2021coupledCPD_DOAestimation}.
CTDs are also important in methods based on high-order statistics for blind system identification~\cite{vaneeghem2018coupledTensorDecBlindSysIdentification}, joint blind source separation from multiple datasets~\cite{gong2018jointBSS_coupledCPD}, and for blind deconvolution of underdetermined systems~\cite{yu2007parafacBlindDecomvolutionMIMO}. 
Moreover, related approaches coupling both matrix and tensor data have shown great potential in metabolomics, combining nuclear magnetic resonance and spectrometry or spectroscopy measurements~\cite{acar2015dataFusionMetabolomics}, and in neuroimaging, fusing electroencephalogram (EEG) with functional magnetic resonance imaging (fMRI) data~\cite{karahan2015tensor_EEG_fMRI_fusion,chatzichristos2022fusionEEGfMRI_doubleCoupledMatrixTensor}, or fMRI with magnetoencephalography (MEG) data~\cite{belyaeva2024learningBrainDynamica_MEG_fMRI}.
\cred{The flexibility of CTDs allow them to address data fusion both in the \emph{multiset} setting, where data is collected using the same modality over different conditions or observation times (e.g., fMRI datasets of different subjects), and also in the \emph{multimodal} setting, where information is collected about the same phenomenon using different modalities or sensors (e.g., EEG and fMRI datasets)~\cite{adali2018ica_iva_dataFusionOverview}.}

One of the main advantages of CTDs is that the coupled decomposition can be unique even when uniqueness does not hold for the decomposition of any of the individual tensors and matrices~\cite{lathauwer2017coupledMatrixTensorFactPartially}.
This occurs, for example, when coupling images with different spectral and spatial resolutions~\cite{kanatsoulis2018hyperspectralSRR_coupledCPD}. 
\cred{Thus, CTDs can exploit advantages of each data modality, being able to leverage weaker uniqueness results of each dataset to provide uniqueness to the full decomposition.}

\textbf{Shared and distinct features:}
In many applications, the following aspect of CTDs is important: even if the modes of all coupled tensors express the same ``view'' of the data (e.g., spatial dimensions, subjects, time, etc.), there is often content unique to each dataset in addition to the content shared with other measurements.
This can occur due to datasets being captured from multiple subjects or from experiments running in different conditions, which provides more diversity but also incurs factors of variation which do not occur in all measurements~\cite{smilde2017commonDistinctComponentsFusion}.
For instance, in multitask fMRI data fusion this allows one to account for task-specific information~\cite{borsoi2022coupledCP_dec_fMRI}. 
Moreover, accurately identifying underlying shared and distinct factors from EEG and fMRI data reveals biologically meaningful components (biomarkers) which allow for differentiating between patients with schizophrenia and healthy controls~\cite{acar2019biomarkersStructureRevealingTensorFusion}.
Subspace-based representation learning of face images can be made more accurate by considering one subspace common to all persons along with others unique to each individual~\cite{prince2007probabilisticLDA_sharedCommon}. Moreover, when fusing hyperspectral images (HSI) with multispectral images (MSI) of different spatial/spectral resolutions,
such models account for the distinct information present in each image due to differences between their acquisition conditions
(e.g., cloud cover, illumination or atmosphere)~\cite{prevost2022LL1_HSR,borsoi2020tensorHSRvariability,wang2023deepFusionInterImageVariability}, addressing the so-called inter-image variability problem.
In what follows, we consider a general framework of so-called \emph{personalized} decompositions, which include some components that are common to all datasets, and other which are distinct to each dataset.

Traditional CTDs fail to consider shared and distinct features in the datasets, which motivated the development of more flexible models.
For example, in \cite{acar2014structureRevealingFusion} the authors proposed a flexible decomposition framework in which the contribution of the coupled components to each dataset is weighted by a set of sparse coefficients. This allows components not to contribute to some datasets. 
Other works proposed to ``un-couple'' a subset of the columns of one of the factor matrices. For instance, the authors in \cite{liu2013miningTensorsCommonDistriminative} perform nonnegative tensor factorization with common and distinct parts in one factor to achieve representation learning for classification tasks. A similar model was used in \cite{wen2016tensorSharedUniqueEventAnalytics} for the analysis of critical events in data (from, e.g., social networks).
More recently, this approach has been considered for coupled matrix decomposition~\cite{sorensen2020factorizationSharesUnshared}, and for the coupled block term decomposition (BTD) in multimodal image fusion~\cite{prevost2022LL1_HSR,borsoi2020tensorHSRvariability}. 
A different approach uses flexible couplings, in which the factor matrices of different tensors or matrices are not constrained to be equal, but only to belong to a Euclidean ball~\cite{schenker2020flexibleCouplingsTensorFact,farias2016tensorFusionFlexibleCouplings}.
However, previous approaches suffer from one of two limitations: they either lack flexibility in the observation model or do not have established uniqueness guarantees. In this work we will consider a general personalized CTD framework, addressing both these issues.

\textbf{Uniqueness of CTDs:}
Uniqueness of matrix and tensor decompositions is of fundamental importance in most applications as it allows the recovered factors to be directly interpreted when there is a match between the decomposition and the data generating model~\cite{adali2022reproducibilityReviewSPM}.
The uniqueness of the coupled canonical polyadic decomposition (CPD) was first studied in~\cite{sorensen2015coupledCPD_part1_unique,sorensen2015coupledCPD_part2_algs}. The authors derived uniqueness guarantees based on the Kruskal condition for the case when one of the factors was coupled for all tensors.

More recent works established uniqueness for more flexible data acquisition models and different tensor decompositions.
The recovery of an order-3 tensor admitting a CPD from measurements acquired according to sampling schemes acting separately on each mode (such as sampling entries, fibers or slabs) was studied in~\cite{kanatsoulis2019tensorCompletionNyquistSampling}. Conditions for the recovery of the full tensor were given depending on the uniqueness of individual tensors in the decomposition. %
This approach is closely related to~\cite{kanatsoulis2018hyperspectralSRR_coupledCPD} in the context of hyperspectral and multispectral image fusion, and to~\cite{zhang2020spectrumCartographyCoupledBTD} for spectrum cartography in the case when the observed tensors have missing data, in which a BTD model is considered. However, these problems do not require the uniqueness of all factors of the decomposition; instead, only the recoverability of the unobserved ``high-resolution'' tensor is necessary.

Although the basic problem of fusing HSIs and MSIs does not require the uniqueness of all factor matrices, recent work addressed the inter-image variability problem by considering flexible CTDs containing both shared and distinct factors.
In this case, conditions ensuring the recoverability of the true high resolution image tensor based on degraded measurements were obtained for image tensors admitting both an LL1 BTD~\cite{prevost2022LL1_HSR} and a Tucker decomposition~\cite{borsoi2020tensorHSRvariability}. However, those works assume a very specific measurement model which also contains no variability (i.e., distinct components) for one of the measured tensors.

Other related work studied the recoverability of principal component analysis~\cite{shi2022personalizedPCA_sharedUnique} and dictionary learning~\cite{liang2023personalizedDictionaryLearning_sharedUnique,shi2023heterogeneousMatrixFactorization_sharedUnique} with shared and dataset-specific factors using incoherence assumptions between the shared and distinct components.
The uniqueness of coupled decomposition with structured matrices accounting for shared and unique factors was studied in~\cite{sorensen2020factorizationSharesUnshared}. %

\textbf{Contribution:} In this paper, we propose a general personalized CTD framework for multimodal data fusion. 
A distinctive feature of our approach is that we clearly show how it is possible to leverage the so-called {\em uni-modal uniqueness} of each tensor~\cite{guo2012partialUniquenessPARAFAC,domanov2013uniquenessCPS_part1_oneFactor}, which can be satisfied under milder conditions when compared to the traditional uniqueness requirements, to provide uniqueness guarantees for the full decomposition.
The main contributions of this paper are:

1) Our coupling framework generalizes several existing ones, including~\cite{kanatsoulis2018hyperspectralSRR_coupledCPD,kanatsoulis2019tensorCompletionNyquistSampling} (which do not have distinct components), \cite{borsoi2020tensorHSRvariability,prevost2022LL1_HSR}, which uses a very specific measurement model and have a distinct component in only one dataset. 
\cred{Although~\cite{sorensen2020factorizationSharesUnshared} provides uniqueness results for the CPD with shared and distinct components, it does not consider a general measurement/degradation model or heterogeneity between datasets. However, the results in~\cite{sorensen2020factorizationSharesUnshared} are based on a different low-rank model and cannot be directly compared to the results presented in this paper.}

2) The general multilinear measurement model comprises a wide range of possible practical applications in, e.g., hyperspectral and neuroimaging (e.g., fMRI, MRI, MEG) data fusion problems.

3) Uniqueness results which are easy to interpret and provide insight on the role that ``weaker'' (uni-mode) uniqueness of each measured tensor and the measurement model have on the uniqueness of the full decomposition. An important feature of the proofs is that they are constructive and directly motivate a semi-algebraic decomposition method.

4) Semi-algebraic (motivated by the constructive uniqueness results) and optimization-based algorithms are proposed to compute the decomposition. \cred{The solution provided by the semi-algebraic algorithm also serves as a principled approach to initialize optimization-based methods.}

A general measurement model is presented in Section~\ref{sec:prob_statement} where each measured tensor is represented as a sum of two terms. The first term is linked to a \textit{common component} through an arbitrary multilinear measurement/degradation model, which couples the different datasets. The second (uncoupled) component is \textit{distinct} to each measured tensor representing personalized or dataset-specific information. The common and distinct tensors are assumed to admit CPDs.
In Section~\ref{sec:uniqueness_full} we demonstrate that this decomposition is generically unique (i.e., the common and distinct components can be recovered) under mild conditions. The developed uniqueness conditions are also highly interpretable: informally, they show that recovery of the common and distinct components is possible as long as 1) one of the low-resolution measured tensors is fully unique, and 2) for each mode $j$, there exists a measured tensor that is mode-$j$ unique and whose mode-$j$ factor is measured at full resolution. Obtaining uniqueness of the full decomposition from ``weaker'' uniqueness results of each individual tensor highlights the benefits of data fusion.
A semi-algebraic algorithm inspired by the uniqueness results as well as an optimization approach based on \cred{alternating least squares (ALS)} are proposed in Section~\ref{sec:computing_the_decomposition}.
Experiments with synthetic data and real hyperspectral images illustrate the advantage of the proposed framework compared to competing algorithms.
\cred{In addition, a preliminary conference version of this work containing experimental results for multi-subject and multitask fMRI data fusion based on a more constrained version of the model proposed in Section~\ref{sec:prob_statement} was presented in~\cite{borsoi2022coupledCP_dec_fMRI}. These results further demonstrate the practical value of the proposed model.}

\section{Background}

\subsection{Definitions and notation}

Scalars, vectors and matrices are denoted by plain font ($x$ or $X$), lowercase bold font ($\bx$) and uppercase bold font ($\bX$), respectively. Order-3 tensors are represented by calligraphic \cred{bold} font ($\tensor{X}$). 
The $(i,j)$-th element of a matrix $\bX$ is denoted by $[\bX]_{i,j}$, while the $(i,j,k)$-th element of a tensor $\tensor{X}$ is denoted by $[\tensor{X}]_{i,j,k}$.
\cred{We also use the notation $a:b$, for $a$ and $b$ being positive integers to compute submatrices of a given matrix. For example, $[\bX]_{1:a,1:b}$ denotes the first $a$ rows and $b$ columns of $\bX$.}
We denote by $\otimes$ and by $\odot$ the Kronecker and the Khatri-Rao products \cite{kolda2009tensor}.
The outer product between three vectors is defined as $\ba\circ\bb\circ\bc$, $[\ba\circ\bb\circ\bc]_{i,j,k}=a_ib_jc_k$, and constitutes a rank-1 tensor.
We denote the left pseudoinverse of matrix $\bX$ by $\bX^{\dagger}$. 
A mode-$k$ fiber of a tensor $\tensor{X}$ is the vector obtained by fixing all but one mode of $\tensor{X}$, while a slice is a two dimensional subset obtained by fixing one mode of $\tensor{X}$. For $K\in\amsmathbb{N}_{+}$, we also use $[K]$ to represent the set of integers $\{1,\ldots,K\}$.
\cred{Note that although we consider the case of third order tensors to keep the presentation cleaner, the proposed framework can be extended to higher order tensors as well.}

\begin{definition}
The Kruskal rank (also called K-rank) of a matrix $\bX$, denoted by $\kr(\bX)$, is the largest number $r$ such that \cred{every set} of $r$ columns of $\bX$ are linearly independent. 
\end{definition}

\begin{definition}
The mode-$k$ matricization of a tensor $\tensor{T}\in\amsmathbb{R}^{N_1\times N_2\times N_3}$, denoted by $\unfold{\bT}{k}$, arranges its mode-$k$ fibers to be the rows of matrix $\unfold{\bT}{k}\in\amsmathbb{R}^{N_{\ell}N_{m}\times N_k}$ such that the $n_k$-th column of $\unfold{\bT}{k}$ consists of the vectorization of the slice of $\tensor{T}$ obtained by fixing the index of the $k$-th mode of $\tensor{T}$ as $n_k$.
\end{definition}

\begin{definition}
The mode-$k$ product between a tensor $\tensor{T}$ and a matrix $\bB$ is denoted by $\tensor{T}\times_k\bB$ and consists of multiplying every mode-$k$ fiber of $\tensor{T}$ by $\bB$. It can be written using the mode-$k$ matricization as $\tensor{U}=\tensor{T}\times_k\bB\Leftrightarrow\unfold{\bU}{k}=\unfold{\bT}{k}\bB^\top$.
\end{definition}

\begin{definition}
The full multilinear product consists of mode-$k$ products between a tensor $\tensor{T}$ and matrices $\bB_k$ for each mode $k\in\{1,2,3\}$, and can be expressed as $\ldbrack\tensor{T};\bB_{1},\bB_{2},\bB_{3}\rdbrack=\tensor{T}\times_1\bB_{1}\times_2\bB_{2}\times_3\bB_{3}$.
\end{definition}

\subsection{Canonical Polyadic Decomposition}

The polyadic decomposition expresses a tensor $\tensor{X}\in\amsmathbb{R}^{I\times J\times K}$ as a sum of rank-1 terms:
\begin{align}
    \tensor{X} &= \sum_{r=1}^R \ba_r \circ \bb_r \circ \bc_r 
    = \ldbrack \bA,\bB,\bC \rdbrack \,,
\end{align}
where matrices $\bA=[\ba_1,\ldots,\ba_R]$, $\bB=[\bb_1,\ldots,\bb_R]$, $\bC=[\bc_1,\ldots,\bc_R]$ are called the factor matrices. When $R$ is minimal, this is called the canonical polyadic decomposition (CPD), and $R$ is called the (CP) rank of $\tensor{X}$.

A fundamental property of the CPD is that it is \emph{essentially unique} under mild conditions. More precisely, a decomposition $\tensor{X}=\ldbrack\bA,\bB,\bC\rdbrack$ is essentially unique if every alternative decomposition $\tensor{X}=\ldbrack\bA',\bB',\bC'\rdbrack$ satisfies:
\allowdisplaybreaks
\cred{\begin{align}
    \bA'=\bA\bPi\bLambda_A \,,%
    \label{eq:def_ess_unique_a}\\
    \bB'=\bB\bPi\bLambda_B \,,%
    \label{eq:def_ess_unique_b}\\
    \bC'=\bC\bPi\bLambda_C \,,
    \label{eq:def_ess_unique_c}
\end{align}
where} $\bPi$ is a permutation matrix and $\bLambda_A$, $\bLambda_B$, $\bLambda_C$ are diagonal matrices satisfying $\bLambda_A\bLambda_B\bLambda_C=\bI$. This means that the factors of the decomposition are unique up to a permutation and scaling of their columns. \cred{We also say that one of the factor matrices of the CPD, say, $\bC$, is essentially unique if~\eqref{eq:def_ess_unique_c}
is satisfied, but nothing is required of the remaining factor matrices $\bA$ and $\bB$.}

Several works investigated the uniqueness of the CPD. One of the earliest conditions was obtained by Kruskal:
\begin{theorem}\cite{kruskal1977uniquenessTensor,stegeman2007kruskalConditionAcessible} \label{thm:kruskals}
The CPD of an order-3 tensor $\tensor{X}=\ldbrack\bA,\bB,\bC \rdbrack$ with rank $R$ is essentially unique if
\begin{align}
    \kr(\bA) + \kr(\bB) + \kr(\bC) \geq 2R+2 \,.
    \label{eq:kruskal_condition_classical}
\end{align}
\end{theorem}

Note that Kruskal's condition is sufficient but not necessary for a CPD to be unique. Recent work has investigated more relaxed uniqueness conditions~\cite{domanov2013uniquenessCPS_part2_overallDec}. 
In particular, it has been shown that conditions for the CPD of generic tensors to be unique can be mild, especially for tensors that are ``tall'' in at least one of the modes~\cite{chiantini2012genericIdentifiabilityTensors}.
We also note that CPD can be computed algebraically even when the factor matrices are not full rank~\cite{domanov2014CPD_uniqueness_reductionGEVD}. 
We refer to a tensor which satisfies Kruskal's uniqueness condition as \textit{fully unique} in order to specify the case in which all factor matrices (from the three modes) are unique up to permutation and scaling ambiguities and to clearly distinguish it from the uni-mode uniqueness condition discussed in the following.

\subsection{Coupled decompositions and uni-mode uniqueness}
\label{ssec:uni_mode_unique}

Given a set of tensors $\tensor{X}_t$, $t=1,\ldots,T$, coupled decompositions represent each $\tensor{X}_t$ using a low-rank model but constrain some of the factors to be the same for different $t$. For instance, a coupled PD of $\{\tensor{X}_t\}_t$ can be expressed as
\begin{align}
    \tensor{X}_t = \ldbrack \bA_t, \bB_t, \bC \rdbrack \,, \quad t=1,\ldots,T \,,
    \label{eq:rev_coupled_CPD}
\end{align}
where mode-1 and mode-2 factors $\bA_t\in\amsmathbb{R}^{I_t\times R}$ and $\bB_t\in\amsmathbb{R}^{J_t\times R}$ are individual to each tensor $\tensor{X}_t$, while the mode-3 factor $\bC\in\amsmathbb{R}^{K\times R}$ is shared among all tensors.

The uniqueness of decompositions in the form of~\eqref{eq:rev_coupled_CPD} was studied in~\cite{sorensen2015coupledCPD_part1_unique,sorensen2015coupledCPD_part2_algs}. It was shown that conditions for the coupled decomposition to be unique can be much milder than those guaranteeing the uniqueness of every tensor. 
This is an important advantage of coupled decompositions, which can be unique even when some of $\{\tensor{X}_t\}_t$ are matrices~\cite{acar2013dataFusionThroughCoupledFactorizations}.

Such relaxed uniqueness conditions are strongly related to the so-called \textit{uni-mode uniqueness} of tensors, which guarantee the unique recovery of only a single factor of the decomposition.
Such results will be important to guarantee uniqueness of the coupled decompositions later in the paper.
The uni-mode uniqueness of the CPD has been studied in~\cite{guo2012partialUniquenessPARAFAC,domanov2013uniquenessCPS_part1_oneFactor}. We recall the condition obtained in~\cite{guo2012partialUniquenessPARAFAC}:

\begin{theorem}\cite[Theorem~2.1]{guo2012partialUniquenessPARAFAC} \label{thm:partial_uniq_cp}
Given the CP decomposition of an order-3 tensor $\tensor{X}=\ldbrack\bA,\bB,\bC \rdbrack$ with rank $R$, if $\bC$ has no zero columns and
\begin{align}
    \kr(\bA) + \kr(\bB) + \rank(\bC) \geq 2R+2
    \label{eq:uni_mode_unique_cp1}
\end{align}
holds, then the third factor matrix $\bC$ is essentially unique.
\end{theorem}

Other conditions for the uni-modal and full uniqueness of the CPD were later studied considering the compound matrix framework in~\cite{domanov2013uniquenessCPS_part1_oneFactor,domanov2013uniquenessCPS_part2_overallDec}. In particular, the uni-mode uniqueness condition given in Expression~1.12 of~\cite{domanov2013uniquenessCPS_part1_oneFactor} shows that, under some mild additional conditions on factors $\bA$ and $\bB$, the condition~\eqref{eq:uni_mode_unique_cp1} from Theorem~\ref{thm:partial_uniq_cp} can be substituted by the alternative condition $\rank(\bC)+\min(\kr(\bA),\kr(\bB))\geq R+1$.

The uni-mode uniqueness of the CPD plays a fundamental role in the uniqueness of coupled decomposition models.
Also, the uni-mode uniqueness results for a single CPD have been extended to the \cred{coupled case~\eqref{eq:rev_coupled_CPD}} to provide more relaxed conditions to recover the factor $\bC$ \cite[Theorem 4.6]{sorensen2015coupledCPD_part1_unique}. 
Moreover, the coupled \cred{CPD~\eqref{eq:rev_coupled_CPD}} can be transformed into an LL1 BTD model (see~\cite[Section~4.3]{sorensen2015coupledCPD_part1_unique}). \cred{Thus, \eqref{eq:rev_coupled_CPD}} inherits the uniqueness of the LL1 BTD~\cite{lathauwer2008tensor_BTD2_uniqueness}, which guarantees the recovery of $\bC$.
These results on the uniqueness of $\bC$ serve as the basis to derive uniqueness guarantees for the coupled \cred{CPD~\eqref{eq:rev_coupled_CPD}} even when none of the individual tensors is unique~\cite{sorensen2015coupledCPD_part1_unique}.
Thus, uni-mode uniqueness has a significant impact in the development of recovery conditions in data fusion. In particular, it is possible to recover a ``high-resolution'' tensor from the coupled decomposition of multiple ``low-resolution'' ones, where the low-resolution tensors only need to ensure the uniqueness of a high-resolution factor matrix by being uni-mode unique, but not necessarily fully unique.

\section{Personalized coupled tensor decomposition}
\label{sec:prob_statement}

We consider the problem of recovering a common order-3 tensor $\tensor{C}\in\amsmathbb{R}^{M_1\times M_2\times M_3}$ and distinct order-3 tensors $\tensor{D}_k\in\amsmathbb{R}^{N_{k,1}\times N_{k,2}\times N_{k,3}}$ from a set of possibly degraded measurements $\tensor{Y}_k\in\amsmathbb{R}^{N_{k,1} \times N_{k,2} \times N_{k,3}}$, $k=1,\ldots,K$, which are acquired according to the following measurement model:
\begin{align}
    \tensor{Y}_k = \mathscr{P}_k(\tensor{C}) + \tensor{D}_k \,,
    \label{eq:meas_model}
\end{align}
where $\mathscr{P}_k$ is a separable operator of the form:
\begin{align}
    \mathscr{P}_k(\tensor{C}) = \tensor{C} \times_1 \bP_{k,1} \times_2 \bP_{k,2} \times_3 \bP_{k,3} \,,
    \label{eq:degrad_model}
\end{align}
with $\bP_{k,j}\in\amsmathbb{R}^{N_{k,j}\times M_j}$ being matrices of rank~$\min\{N_{k,j},M_j\}$ which describe the way in which the $j$-th mode of $\tensor{C}$ is measured in (or contributes to) $\tensor{Y}_k$.
The common tensor $\tensor{C}$ represents information/content that is shared among all datasets, linking them together, while the distinct \cred{tensors $\tensor{D}_k$ represent} information unique to each dataset.
The model in~\eqref{eq:meas_model} and~\eqref{eq:degrad_model} is illustrated in Figures~\ref{fig:coupled_model} and~\ref{fig:degradation}.

This general model can be related to different applications. For instance, in multimodal image fusion, tensors $\tensor{Y}_k$ can represent multichannel images acquired from the same scene, $\tensor{C}$ is a latent high-resolution image, operators $\bP_{k,j}$ represent spatial and spectral degradations during image acquisition, and $\tensor{D}_k$ contains inter-image variations/changes~\cite{prevost2020coupledTucker_hyperspectralSRR_TSP,borsoi2020tensorHSRvariability,prevost2022LL1_HSR}.
When $\tensor{Y}_k$ represents fMRI data from multiple tasks, $\tensor{C}$ contains functional brain networks that are active during data acquisition for all tasks, whereas $\tensor{D}_k$ contains task-specific functional networks~\cite{borsoi2022coupledCP_dec_fMRI}.
This model could also be used, e.g., to separate images of faces from multiple subjects (represented in $\tensor{Y}_k$) represent common ($\tensor{C}$) and subject specific ($\tensor{D}_k$) components as in~\cite{prince2007probabilisticLDA_sharedCommon}, or to perform fusion of EEG and fMRI data (represented in $\tensor{Y}_k$) accounting for components that might be present only in a single data modality ($\tensor{D}_k$)~\cite{acar2019biomarkersStructureRevealingTensorFusion,acar2014structureRevealingFusion}.

Our aim in this work is to recover $\tensor{C}$ and $\{\tensor{D}_k\}_k$ from $\{\tensor{Y}_k\}_k$ by considering a low-rank CPD tensor model for $\tensor{C}$ and $\tensor{D}_k$, which is described in the following.
Assuming that the common and distinct components admit CPDs, we can write the tensors in~\eqref{eq:meas_model} and~\eqref{eq:degrad_model} as:
\begin{align}
    \tensor{C} & = \ldbrack \bC_{1}, \bC_{2}, \bC_{3}\rdbrack \,,
    \label{eq:cp_z}
    \\
    \tensor{D}_k & = \ldbrack \bD_{k,1}, \bD_{k,2}, \bD_{k,3} \rdbrack \,,
    \label{eq:cp_psi}
    \\
    \mathscr{P}_k(\tensor{C}) & = \ldbrack  \bP_{k,1}\bC_{1}, \bP_{k,2}\bC_{2}, \bP_{k,3}\bC_{3}\rdbrack \,, 
    \label{eq:cp_Pz}
\end{align}
where $\bC_{j}\in\amsmathbb{R}^{M_j\times R}$ and $\bD_{k,j}\in\amsmathbb{R}^{N_{k,j}\times L_{k}}$, $j\in\{1,2,3\}$ are the factor matrices.
\cred{Note that the model in~\eqref{eq:cp_Pz} consists in a polyadic decomposition with linear constraints on the factor matrices. This is directly connected to the CANDELINC decomposition~\cite{douglas1980candelinc}. However, unlike the traditional \mbox{CANDELINC}, we do not require matrices $\bP_{k,j}$ to be full column rank. In such cases, recovery of the common factors $\bC_j$ can still be achieved in the present case due to the coupling between the different tensors.}
The tensor $\tensor{Y}_k$ also admits a polyadic decomposition with factor matrices equal to
\begin{align}
    \big[\bP_{k,j}\bC_{j}, \, \bD_{k,j} \big] \,,
    \label{eq:cp_yk}
\end{align}
for $j\in\{1,2,3\}$. \cred{Note, however, that although $\tensor{C}$ and $\tensor{D}_k$ are assumed to admit CPDs, the polyadic decomposition of $\tensor{Y}_k$ with the factor matrices given in~\eqref{eq:cp_yk} is not necessarily canonical, since the number of components might not be minimal.} This model is fairly general, and encompasses as particular cases several applications as illustrated by the following examples.

\begin{figure}[t]
    \centering
    \includegraphics[width=0.65\linewidth]{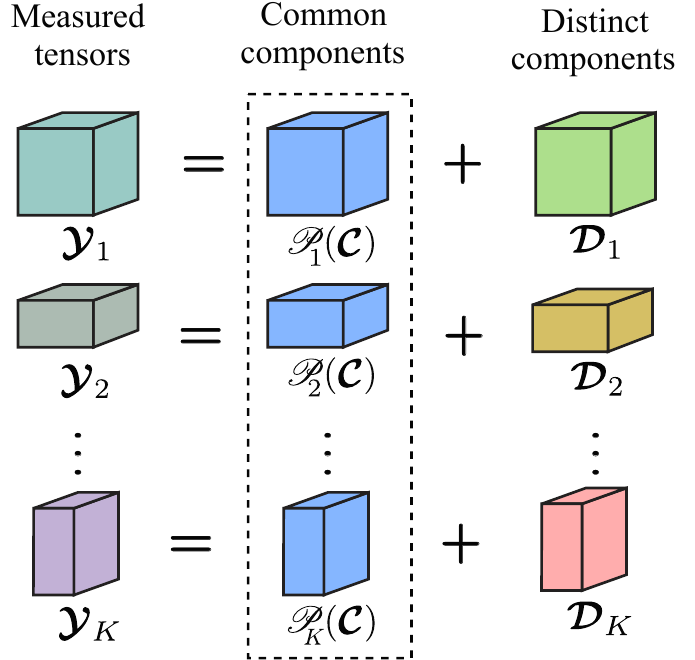}
    \vspace{-1.5ex}
    \caption{Illustration of the model in~\eqref{eq:meas_model}. A latent tensor $\tensor{C}$ is common in all measurements and acquired through an operator $\mathscr{P}_k$ while \cred{tensors} $\tensor{D}_k$ are distinct to each measurement, leading to a personalized decomposition.}
    \label{fig:coupled_model}
\end{figure}

\begin{figure}[t]
    \centering
    \includegraphics[width=0.5\linewidth]{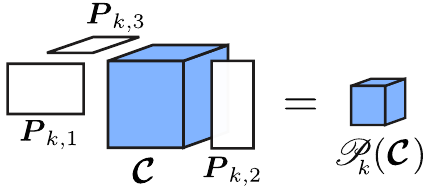}
    \vspace{-1.7ex}
    \caption{Illustration of the multilinear measurement model in~\eqref{eq:degrad_model}.}
    \label{fig:degradation}
\end{figure}

\begin{example}(fMRI decomposition) \label{ex:fMRIexample}
Multi-subject fMRI data of the brain can be represented in the form of spatial feature maps, which are obtained by performing voxel-wise regression of the raw time-series data~\cite{levin2017quantifyingInteractionMultipleDatasetsSchizophrenia}. Such feature representations can be ordered as tensors, whose modes correspond to the spatial dimensions (voxels), to the subject index, and to sets of features corresponding to different data acquisitions within a task. By acquiring data while subjects are performing certain tasks (e.g., recognizing number sequences displayed on a screen, or sound tunes), different tensors can be constructed, one for each task. Then, by decomposing them according to the model~\eqref{eq:meas_model},~\eqref{eq:degrad_model}, with the matrices $\bP_{k,j}$ being identity matrices, it is possible to recover the brain activations which are common to all tasks in $\tensor{C}$, and those which are specific to each of the tasks in $\tensor{D}_k$~\cite{borsoi2022coupledCP_dec_fMRI}. %

\end{example}

\begin{color}{red}
\begin{remark}
    Although the fMRI voxels are matricized in Example~\ref{ex:fMRIexample} (i.e., they are considered a single mode of the tensor), each of the three spatial dimensions of the fMRI volume can be considered as additional tensor modes, as investigated in~\cite{chatzichristos2019fMRI_unmixing_tensor}. Although this would require extending the proposed model to higher order tensors (the measurements would have order 5), it allows the use of more general low-rank models which can reduce the amount of unknowns to be estimated.
\end{remark}
\end{color}

\begin{example}(MRI fusion)
A high-resolution MRI tensor $\tensor{C}$ (whose modes represent the three spatial dimensions) can be computed efficiently by fusing three acquisitions, each with a smaller resolution (i.e., a smaller number of slices) in one of the modes, which can make acquisition time smaller~\cite{prevost2023nonLocalSuperResolutionMRI,prevost2022superResolutionMRItensorTucker}. The acquisition model can be written as in~\eqref{eq:meas_model},~\eqref{eq:degrad_model}, in which $\bP_{k,j}$ consist in either the identity matrix (for modes observed in full resolution), or downsampling and reweighing matrices (for modes observed in lower resolution). The variability between the acquisitions, represented in $\tensor{D}_k$, can originate from patient motion, which might not be perfectly compensated between the different acquisitions~\cite{shi2022motionFetalMRI}.
\end{example}

\begin{example}(Spectral image fusion)
\label{ex:HS_MS_fusion}
Spectral images can be represented as order-3 tensors, with two modes corresponding to spatial dimensions (pixels), and one mode to the spectral dimension (bands). 
A high resolution image $\tensor{C}$ can be recovered by fusing two images, one with low spectral resolution (an MSI), and one with low spatial resolution (an HSI)~\cite{kanatsoulis2018hyperspectralSRR_coupledCPD}. These images can be represented as degraded versions of $\tensor{C}$ according to model~\eqref{eq:meas_model},~\eqref{eq:degrad_model}, where matrices $\bP_{k,j}$ are either identity matrices (for the modes observed with high resolution), or degradation matrices represent the spatial or spectral responses of the sensors in a given mode. Inter-image variability between the different acquisitions, represented in $\tensor{D}_k$, originates from illumination or other acquisition differences between the images~\cite{borsoi2020tensorHSRvariability,borsoi2021spectralVariabilityReview}.
\end{example}

\cred{The model in~\eqref{eq:meas_model}--\eqref{eq:cp_Pz} considers that all factors of $\tensor{C}$ are coupled among the different datasets through the operators $\mathscr{P}_k$. However, when fusing heterogeneous datasets such as EEG and fMRI data~\cite{acar2017tensor_EEG_fMRIfusionSchizophrenia} or simulated and real metabolomics datasets~\cite{li2024longitudinalMetabolomicsTensor}, where some tensor modes such as the time dimension might not be necessarily aligned among datasets, a more flexible approach in which such modes of the tensor are left ``un-coupled'' (i.e., the coupling is enforced only in one or two of the modes) might be desirable. We will introduce such additional flexibility when presenting the optimization-based CTD framework later in Section~\ref{ssec:optim_bcd_sol}.}

In the following, we first study the uniqueness of this decomposition \cred{in Section~\ref{sec:uniqueness_full}}. Then, \cred{in Section~\ref{sec:computing_the_decomposition}} we develop two algorithms to compute \cred{it}, one being semi-algebraic, and another being based on an alternating optimization approach.

\section{Uniqueness of the decomposition}
\label{sec:uniqueness_full}

In this section, we provide both deterministic and generic conditions under \cred{which the proposed coupled decomposition model is unique, allowing the recovery of the common and distinct tensors $\tensor{C}$ and $\{\tensor{D}_k\}_k$ from the measurements $\{\tensor{Y}_k\}_k$.}
Those results will also motivate the development of a semi-algebraic algorithm to compute the decomposition.

\subsection{Deterministic recoverability conditions}

Considering the uniqueness and uni-mode uniqueness results for the CPD presented in Section~\ref{ssec:uni_mode_unique}, the following result concerning the recoverability of the proposed coupled CPD-based model from Section~\ref{sec:prob_statement} can be stated.

\begin{theorem}
\label{thm:uniqueness_deterministic}
Consider the measurement model \eqref{eq:meas_model}, \eqref{eq:degrad_model} \cred{and the CPD model~\eqref{eq:cp_z}, \eqref{eq:cp_psi}, \eqref{eq:cp_Pz},~\eqref{eq:cp_yk}.} Then, if the following conditions are satisfied:

\begin{description}
\item[A1:] There exists an $\eta\in[K]$ such that $\tensor{Y}_{\eta}$ is fully unique.

\item[A2:] For each ${j\in[3]}$ there exists $\xi_j\in[K]$ such that $\tensor{Y}_{\xi_j}$ is mode-$j$ unique and $\bP_{\xi_j,j}$ has full column rank.

\item[A3:] There exists at least one $j\in[3]$ such that $\xi_j\neq\eta$. Moreover, for all such $j$, the Kruskal rank of matrix $\big[\bP_{\eta,j}\bC_{j},\,\bP_{\eta,j}\bP_{\xi_j,j}^\dagger\bD_{\xi_j,j},\,\bD_{\eta,j}\big]$ is bigger than one.

\end{description}
\cred{then the common and distinct tensors $\tensor{C}$ and $\{\tensor{D}_k\}_{k}$ can be uniquely recovered from $\{\tensor{Y}_k\}_{k}$.}
\end{theorem}

\begin{proof}

\medskip
\cred{Consider a decomposition for the coupled CPD model in~\eqref{eq:meas_model}--\eqref{eq:cp_yk} given by $\{\bC_j,\bD_{k,j}\}_{k,j}$. Suppose that there exists an alternative decomposition $\big\{\mywidehat{\bC}_{i},\mywidehat{\bD}_{k,i}\big\}_{i,k}$ that also satisfies equations~\eqref{eq:meas_model}--\eqref{eq:cp_yk}.} 

\textbf{Step 1 (linking the recovered factors).} 
Let us take a $j\in[3]$ satisfying $\xi_j\neq\eta$. Such a $j$ is guaranteed to exist due to hypothesis A3. The uniqueness of $\tensor{Y}_{\eta}$ and mode-$j$ uniqueness of $\tensor{Y}_{\xi_j}$ from hypotheses A1 and A2 mean that their decompositions \cred{$\big\{\mywidehat{\bC}_{i},\mywidehat{\bD}_{\eta,i}\big\}_i$ and $\big\{\mywidehat{\bC}_{j},\mywidehat{\bD}_{\xi_j,j}\big\}_j$ satisfy:}
\begin{align}
    \big[\bP_{\eta,i}\mywidehat{\bC}_{i},\,\mywidehat{\bD}_{\eta,i}\big] 
    &= \big[\bP_{\eta,i}\bC_{i},\,\bD_{\eta,i}\big] \bPi_{\eta} \bLambda_{\eta,i} \,,
    \label{eq:uniq_cpd3}
\end{align}
which holds for $i\in[3]$, and
\begin{align}
    \big[\bP_{\xi_j,j}\mywidehat{\bC}_{j},\,\mywidehat{\bD}_{\xi_j,j}\big]
    &= \big[\bP_{\xi_j,j}\bC_{j},\,\bD_{\xi_j,j}\big] \bPi_{\xi_j} \bLambda_{\xi_j,j} \,.
    \label{eq:uniq_cpd4}
\end{align}
where $\bPi_{\xi_j}$ and $\bPi_{\eta}$ are permutation matrices, and $\bLambda_{\xi_j,j}$ and $\bLambda_{\eta,i}$ are diagonal matrices, with $\bLambda_{\eta,1}\bLambda_{\eta,2}\bLambda_{\eta,3}=\bI$.

Moreover, hypothesis A2 means that $\bP_{\xi_j,j}$ has full column rank. This allows us to relate the common part of the mode-$j$ factor matrices of the tensor in measurements $\xi_j$ and $\eta$ as
\begin{align}
    \bP_{\eta,j}\mywidehat{\bC}_{j}
    ={}& \big[\bP_{\eta,j}\bP_{\xi_j,j}^\dagger\big] \bP_{\xi_j,j}\mywidehat{\bC}_{j} \,.
    \label{eq:unique_relate_commons_1}
\end{align}
Using~\eqref{eq:uniq_cpd3} and~\eqref{eq:uniq_cpd4} on the left- and right-hand side of equation~\eqref{eq:unique_relate_commons_1}, respectively, we obtain:
\begin{align}
    & \Big(\bP_{\eta,j}\bC_{j}\bPi_{\eta}^{(1)} + \bD_{\eta,j}\bPi_{\eta}^{(2)} \Big) \bLambda_{\eta,j}^{(1)}
    \label{eq:cp_thm_coupl_6}
    \\
    &= \bP_{\eta,j}\bP_{\xi_j,j}^\dagger\Big(\bP_{\xi_j,j}\bC_{j}\bPi_{\xi_j}^{(1)} + \bD_{\xi_j,j}\bPi_{\xi_j}^{(2)} \Big) \bLambda_{\xi_j,j}^{(1)}
    \label{eq:cp_thm_coupl_4}
    \\
    &= \Big(\bP_{\eta,j}\bC_{j}\bPi_{\xi_j}^{(1)} +  \bP_{\eta,j}\bP_{\xi_j,j}^\dagger\bD_{\xi_j,j}\bPi_{\xi_j}^{(2)} \Big) \bLambda_{\xi_j,j}^{(1)} \,,
    \label{eq:cp_thm_coupl_5}
\end{align}
where
\begin{align}
    \bPi_{\rho} &= \begin{bmatrix}
    \bPi_{\rho}^{(1)} & \bPi_{\rho}^{(3)} \\[0.1cm]
    \bPi_{\rho}^{(2)} & \bPi_{\rho}^{(4)} \\
    \end{bmatrix} \,, 
    \hspace{2ex} 
    \bLambda_{\rho,j} = \begin{bmatrix}
    \bLambda_{\rho,j}^{(1)} & \cb{0} \\[0.1cm]
    \cb{0} & \bLambda_{\rho,j}^{(2)} 
    \end{bmatrix} \,,
\end{align}
denote the partitions of $\bPi_{\rho}$ and $\bLambda_{\rho,j}$ into blocks of appropriate dimensions, for $\rho\in\{\xi_j,\eta\}$.
Matrices $\bPi_{\rho}^{(i)}$, $i=1,\ldots,4$ are blocks of a permutation matrix, thus having at most one element equal to one in each row or column, with the remaining ones being equal to zero.

\medskip
\textbf{Step 2 (solving the ambiguities).} In order to proceed, we need to show that~\eqref{eq:cp_thm_coupl_5} yields a permutation that allows for the correct recovery of the common part.  We can now show that the permutations will be block diagonal, and will not mix the common and distinct parts of the factor matrices. Again, let us consider a $j\in[3]$ satisfying hypothesis A3, such that $\xi_j\neq\eta$. Note that the columns of the matrices in the left- and right-hand side of~\eqref{eq:cp_thm_coupl_6}--\eqref{eq:cp_thm_coupl_5} can be expressed as:
\begin{align}
    \alpha (\ba + \bb) &= \beta (\ba'+\bc) \neq \cb{0} \,,
    \label{eq:proportional_simple_thm}
\end{align}
where $\ba$ and $\bb$ are either the zero vector or columns of 
$\bP_{\eta,j}\bC_{j}$ and $\bD_{\eta,j}$, respectively, and, similarly, $\ba'$ and $\bc$ are either the zero vector or columns of $\bP_{\eta,j}\bC_{j}$ and $\bP_{\eta,j}\bP_{\xi_j,j}^\dagger\bD_{\xi_j,j}$. 
Moreover, since $\bPi_{\rho}$ is a permutation matrix, only one vector at each side of~\eqref{eq:proportional_simple_thm} is different from zero; thus, one vector among among $\{\ba,\bb\}$, and one vector among among $\{\ba',\bc\}$ is equal to zero.
Scalars $\alpha$ and $\beta$ are scalings related to $\bLambda_{\eta,j}^{(1)}$ and $\bLambda_{\xi_j,j}^{(1)}$ and different from zero. 

Note that~\eqref{eq:proportional_simple_thm} implies that $\ba+\bb$ and $\ba'+\bc$ are proportional. However, since the Kruskal rank of $\big[\bP_{\eta,j}\bC_{j},\,\bP_{\eta,j}\bP_{\xi_j,j}^\dagger\bD_{\xi_j,j},\,\bD_{\eta,j}\big]$ is bigger than one due to hypothesis A3, \cred{no two columns of the matrix are proportional}, which implies that~\eqref{eq:proportional_simple_thm}
can only be satisfied if $\bb=\bc=\cb{0}$.
Since this holds for all columns of~\eqref{eq:cp_thm_coupl_6}--\eqref{eq:cp_thm_coupl_5}, we have that $\bPi_{\eta}^{(2)}=\cb{0}$ and $\bPi_{\xi_j}^{(2)}=\cb{0}$. This implies that the blocks $\bPi_{\eta}^{(1)}$ and $\bPi_{\xi_j}^{(1)}$ are themselves permutation matrices, and that $\bPi_{\xi_j}$ and $\bPi_{\eta}$ are block diagonal.

\medskip
\textbf{Step 3 (linking the three modes).}
Note that $\tensor{Y}_{\eta}$ is fully unique due to hypothesis A1, and the above results showed that $\bPi_{\eta}$ is block diagonal. We can also show that $\bPi_{\xi_j}$ has a block diagonal structure for $j\in[3]$, not mixing common and distinct parts of the factor matrices. Specifically, if $\xi_j\neq\eta$, then by the previous arguments we have directly that $\bPi_{\xi_j}$ is block diagonal with the desired structure. Otherwise, if $\xi_j=\eta$, then obviously $\bPi_{\xi_j}=\bPi_{\eta}$, thus being block diagonal with the same structure.
It remains to show that the permutations and scalings are consistent across the three modes.

To proceed, consider $\xi_j$, for each $j\in[3]$. Due to the block diagonal structure of $\bPi_{\xi_j}$, the mode-$j$ uniqueness in~\eqref{eq:uniq_cpd4}, the coupling between $\eta$ and $\xi_j$ in mode $j$, we have
\begin{align}
    & \bP_{\xi_j,j}\mywidehat{\bC}_{j} = \bP_{\xi_j,j}\bC_{j} \bPi_{\xi_j}^{(1)} \bLambda_{\xi_j,j}^{(1)} \,,
\end{align}
which, using the full column rank of $\bP_{\xi_j,j}$, leads to the following condition:
\begin{align}
    & \mywidehat{\bC}_{j} = \bC_{j} \bPi_{\xi_j}^{(1)} \bLambda_{\xi_j,j}^{(1)} \,.
    \label{eq:cp_thm_coupl_afterVar_1}
\end{align}
Due to the uniqueness of $\tensor{Y}_{\eta}$, we can use the block diagonal structure of $\bPi_{\eta}$ with equations~\eqref{eq:uniq_cpd3} and~\eqref{eq:cp_thm_coupl_afterVar_1} to obtain
\begin{align}
    \bP_{\eta,j}\mywidehat{\bC}_j & = \bP_{\eta,j}\bC_j\bPi_{\eta}^{(1)}\bLambda_{\eta,j}^{(1)}
    \nonumber \\
    & = \bP_{\eta,j} \bC_{j} \bPi_{\xi_j}^{(1)} \bLambda_{\xi_j,j}^{(1)} \,.
    \label{eq:cp_thm_coupl_afterVar_2}
\end{align}
Note that due to necessary conditions for the uniqueness of the CPD, $\tensor{Y}_{\eta}$ being unique implies that $\bP_{\eta,j}\bC_{j}$ does not have proportional columns~\cite{domanov2013uniquenessCPS_part2_overallDec}.
Since $\bLambda_{\xi_j,j}^{(1)}$ and $\bLambda_{\eta,j}^{(1)}$ are diagonal and $\bPi_{\eta}^{(1)}$ and $\bPi_{\xi_j}^{(1)}$ are permutation matrices, this means that equation~\eqref{eq:cp_thm_coupl_afterVar_2} is only satisfied when $\bPi_{\xi_j}^{(1)}=\bPi_{\eta}^{(1)}$ and $\bLambda_{\xi_j,j}^{(1)}=\bLambda_{\eta,j}^{(1)}$. 
This holds for all $j\in[3]$. Thus, since $\bLambda_{\eta,1}^{(1)}\bLambda_{\eta,2}^{(1)}\bLambda_{\eta,3}^{(1)}=\bI$, the reconstructed tensor satisfies
\begin{align}
    \mywidehat{\tensor{C}} & = \big\ldbrack \mywidehat{\bC}_{1}, \mywidehat{\bC}_{2}, \mywidehat{\bC}_{3} \big\rdbrack
    \nonumber \\
    & = \big\ldbrack \bC_{1}\bPi_{\eta}^{(1)}\bLambda_{\eta,1}^{(1)}, \bC_{2}\bPi_{\eta}^{(1)}\bLambda_{\eta,2}^{(1)}, \bC_{3}\bPi_{\eta}^{(1)}\bLambda_{\eta,3}^{(1)} \big\rdbrack
    \nonumber \\
    & = \tensor{C} \,,
\end{align}
which proves the recovery of $\tensor{C}$. The recovery of $\tensor{D}_k$ follows directly from~\eqref{eq:meas_model}, which concludes the proof.
\end{proof}

\begin{remark}
    Note that the conditions in Theorem~\ref{thm:uniqueness_deterministic} also imply the uniqueness of the factors of the common tensor $\tensor{C}$. Moreover, conditions for the uniqueness of the factors of the distinct tensors $\tensor{D}_k$ can be obtained by analyzing each tensor separately using Kruskal's condition, since $\tensor{D}_k=\tensor{Y}_k-\mathscr{P}_k(\tensor{C})$.
\end{remark}

The conditions in Theorem~\ref{thm:uniqueness_deterministic} are mild and interpretable. Condition A1 is equivalent to one of the measured tensors $\tensor{Y}_{\eta}$ being fully unique. However, it does not require its factors to be measured in full resolution (i.e., $\bP_{\eta,j}$ need not be full column rank). The full uniqueness of one tensor is used to ensure consistent permutations and scalings among the recovered common factors in the three modes.

Intuitively, condition A2 means that for each mode $j\in[3]$ we need one of the measured tensors to be mode-$j$ unique and to have the corresponding degradation matrix $\bP_{i,j}$ with full column rank. This means that the mode-$j$ factor matrix (containing both the common and distinct components) can be recovered in full resolution from this tensor.

Finally, condition A3 is key to allow us to relate factor matrices containing both common and distinct components from two different measured tensors in order to find out which of the components are the common ones.

\subsection{Generic recoverability conditions}

Based on Theorem~\ref{thm:uniqueness_deterministic} we can also derive generic uniqueness results which give explicit conditions based on the rank of the decomposition \cred{(i.e., the ranks of the common and distinct tensors, given by $R$ and $L_k$, respectively)}, on the measurement matrices $\bP_{k,j}$ and on the tensor dimensions that guarantee that a generic CTD 
following the model presented in Section~\ref{sec:prob_statement} is essentially unique. This result is given in the following theorem.

\begin{theorem}
\label{thm:uniqueness_generic}
Consider the measurement model \eqref{eq:meas_model}, \eqref{eq:degrad_model} \cred{and the CPD model~\eqref{eq:cp_z}, \eqref{eq:cp_psi}, \eqref{eq:cp_Pz},~\eqref{eq:cp_yk}}. Assume that all factors $\{\bC_j,\bD_{k,j}\}_{k,j}$ are drawn from some joint absolutely continuous distribution w.r.t. the Lebesgue measure. Then, if the following conditions are satisfied:
\begin{description}
\item[A4:] (uniqueness of $\eta$): There exists an $\eta\in[K]$ such that
\begin{align}
    \hspace{-6ex}
    \sum_{j\in[3]} \min\{\rank(\bP_{\eta,j}),R+L_{\eta}\} & \geq 2(R+L_{\eta})+2 \,.
    \label{eq:condition_hyp_A1}
\end{align}

\item[A5:] (uni-mode uniqueness of $\xi_j$) For each ${j\in[3]}$ there exist a $\xi_j\in[K]$ such that $\bP_{\xi_j,j}$ is full column rank and the following condition is satisfied:
\begin{align}
    & \hspace{-6ex}
    \min(N_{\xi_j,j},\min(M_j,R)+L_{\xi_j}) \,+
    \nonumber \\ 
    & \hspace{-6ex}
    \sum_{i\neq j} \min\{\rank(\bP_{\xi_j,i}),R+L_{\xi_j}\}
    \geq 2(R+L_{\xi_j})+2 \,.
    \label{eq:condition_hyp_A2}
\end{align}

\item[A6:] There exists at least one $j\in[3]$ such that $\xi_j\neq\eta$.

\end{description}
\cred{then the common and distinct tensors $\tensor{C}$ and $\{\tensor{D}_k\}_{k}$ can be uniquely recovered from $\{\tensor{Y}_k\}_{k}$ with probability one.}
\end{theorem}

Before proceeding with the proof of Theorem~\ref{thm:uniqueness_generic}, we need a couple auxiliary results, provided in the following lemmas.

\begin{lemma}
\label{lem:kruskalCat}
Let the elements of matrices $\bX_1\in\amsmathbb{R}^{R_1\times S_1},\ldots,\bX_K\in\amsmathbb{R}^{R_K\times S_K}$ be drawn according to a joint distribution that is absolutely continuous with respect to the Lebesgue measure in $\amsmathbb{R}^{R_1 S_1\cdots R_K S_K}$, and matrices $\bQ_1\in\amsmathbb{R}^{T\times R_1},\ldots,\bQ_K\in\amsmathbb{R}^{T\times R_K}$ be deterministic. If there exist a full row rank matrix $\bE\in\amsmathbb{R}^{\widetilde{T}\times T}$ such that matrices $\bE\bQ_1,\ldots,\bE\bQ_K$ also have full row rank, then the matrix $\bW=[\bQ_1\bX_1,\ldots,\bQ_K\bX_K]$ has Kruskal rank at least $\min\{\widetilde{T},\sum_k R_k\}$ with probability one.
\end{lemma}

\begin{proof}
Let us denote by $\{\bw_1,\ldots,\bw_i\}$ a subset of $i$ columns of matrix $\bW$. Due to matrix $\bE$ having full row rank, if $\{\bE\bw_1,\ldots,\bE\bw_i\}$ are linearly independent vectors, this implies that $\{\bw_1,\ldots,\bw_i\}$ will also be linearly independent. Since this holds for any subset of columns of $\bW$, it implies that $\kr(\bW)\geq \kr(\bE\bW)$.

We now aim to compute the Kruskal rank of $\bE\bW=[\bE\bQ_1\bX_1,\ldots,\bE\bQ_K\bX_K]$. Since $\bE\bQ_i$ is assumed to be of full row rank for all $i$, this implies that if \cred{$\{\bX_1,\ldots,\bX_K\}$ are jointly} absolutely continuous, \cred{$\{\bE\bQ_i\bX_i\}\bigcup\{\bX_j:j\neq i\}$} will be \cred{jointly} absolutely continuous too (see  Lemma~1 of~\cite{kanatsoulis2018hyperspectralSRR_coupledCPD}). 
Thus, noting that the above result holds when applying the transformation $\bE\bQ_i$ to each block of variables, we can proceed recursively for each of the blocks $i=1,\ldots,K$ to show that the elements of matrix $\bE\bW$ are also jointly absolutely continuous. This, on the other hand, implies that $\rank(\bE\bW)=\min(\widetilde{T},\sum_k R_k)$ with probability one. Since the rank and Kruskal rank of generic matrices coincide, the result follows.
\end{proof}

Note that this lemma can be used to lower bound the generic Kruskal rank of the factor matrices of $\tensor{Y}_{k}$.
\begin{corollary}
    \label{lem:kruskalFactorCorollary}
    Suppose $\bC_j$ and $\bD_{k,j}$ are distributed according to a joint absolutely continuous measure, then
    \[\kr([\bP_{k,j}\bC_j, \bD_{k,j}])\geq\min\{\rank(\bP_{k,j}),R+L_{k}\}\,.\]
\end{corollary}
\begin{proof}
    Let us take $\bQ_1=\bP_{k,j}$, $\bQ_2=\bI$, and consider $\bE\in\amsmathbb{R}^{\rank(\bP_{k,j})\times N_{k,j}}$ to be the matrix which selects the $\rank(\bP_{k,j})$ linearly independent rows of $\bP_{k,j}$. Clearly, $\rank(\bE)=\rank(\bE\bQ_1)=\rank(\bE\bQ_2)=\rank(\bP_{k,j})$. Applying Lemma~\ref{lem:kruskalCat} gives the desired result.
\end{proof}

\begin{lemma}
\label{lem:rankCat}
Consider $\bX_1\in\amsmathbb{R}^{M\times R}$ and $\bX_2\in\amsmathbb{R}^{N\times L}$ distributed according to a joint absolutely continuous measure, and $\bQ\in\amsmathbb{R}^{N\times M}$ deterministic and full column rank. Then, for $\bZ=[\bQ\bX_1,\bX_2]$, 
$\rank(\bZ)=\min(N,\min(M,R)+L)$, 
with probability one.
\end{lemma}

\begin{proof}
Consider the square matrix $\widetilde{\bQ}=[(\bQ^{\dagger})^\top, \bQ_{\perp}^\top]^\top$, where $\bQ^{\dagger}$ is the left pseudoinverse of $\bQ$ and $\bQ_{\perp}$ its orthogonal complement. Since $\bQ$ is full column rank, $\widetilde{\bQ}$ is invertible, and thus $\rank(\bZ)=\rank(\widetilde{\bZ})$, where
\begin{align}\label{eq:tildeZ_blktrg}
    \widetilde{\bZ} = \widetilde{\bQ}\bZ = 
    \begin{bmatrix}
        \bX_1 & \widetilde{\bX}_{2,1} \\
        \cb{0} & \widetilde{\bX}_{2,2} 
    \end{bmatrix} \,,
\end{align}
where $\widetilde{\bX}_{2,1}=\bQ^{\dagger}\bX_2$ and $\widetilde{\bX}_{2,2}=\bQ_{\perp}\bX_2$. Note that since $\widetilde{\bQ}$ is deterministic, $\bX_1$, $\widetilde{\bX}_{2,1},$ and $\widetilde{\bX}_{2,2}$ are jointly absolutely continuous. 
We will consider two cases depending on the shape of $\bX_1$.

\emph{Case $R \ge M$}: In this case,
$\bX_1$ is generically full row rank, and thus by results related to the Schur complement~\cite[fact~6.5.6]{bernstein2009matrixBook}
the rank of the matrix is equal to
\[
\rank(\widetilde{\bZ}) = \rank(\bX_1) + \rank(\widetilde{\bX}_{2,2}) \,,
\]
with the right hand side equal generically to $M+\min(N-M,L) = \min(N,M+L)$.

\emph{Case $R \le M$}: Denote $T = \min(N,R+L)$. We know that the rank is bounded by the dimensions of the matrix, i.e., $\rank(\widetilde{\bZ}) \le T$. To prove that the equality holds, we \cred{use the standard argument that is often used for proving generic (holding with probability $1$) properties~\cite{lyu2020nonlinearCCA}.} We first show that \cred{there exists} a matrix $\widetilde{\bZ}_0$ 
with the block-triangular form~\eqref{eq:tildeZ_blktrg} that has 
\cred{rank $T$. Such a matrix can be constructed as}
\begin{align}
\widetilde{\bZ}_0 =
    \begin{bmatrix}
        \bI_T &  \cb{0} \\
        \cb{0} & \cb{0}
    \end{bmatrix} \,.
\end{align}
\cred{
Note that this matrix is still in the same form as~\eqref{eq:tildeZ_blktrg}, although the dimensions of the blocks are different. 
Then the existence of such $\widetilde{\bZ}_0$ implies that $\rank(\widetilde{\bZ}) = T$ with probability $1$.
Below we explain in detail why it is the case (although it can be also deduced from the semicontinuity of the rank function).} 

\cred{By construction, $\det([\widetilde{\bZ}_0]_{1:T,1:T}) = 1 \neq 0$, which is a multivariate polynomial in $\widetilde{\bZ}$. Since the zero set (set of solutions) of a nonzero polynomial equation is necessarily of Lebesgue measure zero~\cite{caron2005zeroPolynimials}, we conclude that} $\det([\widetilde{\bZ}]_{1:T,1:T}) \neq 0$ with probability $1$ for a matrix of the form~\eqref{eq:tildeZ_blktrg}, hence $\rank(\widetilde{\bZ}) = T$ with probability $1$. Combining the two cases gives the desired result.
\end{proof}

Considering the results in Lemmas~\ref{lem:kruskalCat} and~\ref{lem:rankCat}, as well as the deterministic uniqueness result in Theorem~\ref{thm:uniqueness_deterministic}, we can proceed to prove the generic uniqueness case of Theorem~\ref{thm:uniqueness_generic}.

\begin{proof}
\textbf{Step 1 (full uniqueness).} Let us consider the factor matrix of the $j$-th mode of the $k$-th tensor. Due to the joint absolute continuity of $\bC_j$ and $\bD_{k,j}$, we can use Corollary~\ref{lem:kruskalFactorCorollary} to show that the Kruskal rank of the $j$-th factor matrix of $\tensor{Y}_{k}$ is lower bounded as 
\[\kr([\bP_{k,j}\bC_j, \bD_{k,j}])\geq\min\{\rank(\bP_{k,j}),R+L_{k}\}\]
with probability one. 
Thus, for the tensor indexed by $\eta$, condition~\eqref{eq:condition_hyp_A1} of A4 leads to
\begin{align}
    \sum_{j\in[3]} \kr([\bP_{\eta,j}\bC_j, \bD_{\eta,j}]) 
    & \geq \sum_{j\in[3]} \min\{\rank(\bP_{\eta,j}),R+L_{\eta}\} 
    \nonumber \\
    & \geq 2(R+L_{\eta})+2 \,,
\end{align}
which implies that the CPD of $\tensor{Y}_{\eta}$ is essentially unique. This means that condition A1 of Theorem~\ref{thm:uniqueness_deterministic} is satisfied with probability one.

\medskip
\textbf{Step 2 (uni-mode uniqueness).} 
Moreover, for each $\xi_j$, $j\in[3]$, for which $\bP_{\xi_j,j}$ is full column rank due to hypothesis A5, we can use Lemma~\ref{lem:rankCat} to show that, with probability one, the rank of the $j$-th factor matrix of $\tensor{Y}_{\xi_j}$ is equal to 
\[\rank([\bP_{\xi_j,j}\bC_j, \bD_{\xi_j,j}])=\min(N_{\xi_j,j},\min(M_j,R)+L_{\xi_j})\,,\] and from Corollary~\ref{lem:kruskalFactorCorollary} it has no zero columns. This, along with condition~\eqref{eq:condition_hyp_A2} of hypothesis A5, leads to
\begin{align}
    & \rank([\bP_{\xi_j,j}\bC_j, \bD_{\xi_j,j}]) + \sum_{i\neq j} \kr([\bP_{\xi_j,i}\bC_i, \bD_{\xi_j,i}]) 
    \nonumber\\
    & \geq \min(N_{\xi_j,j},\min(M_j,R)+L_{\xi_j})
    \nonumber \\
    & \quad\, + \sum_{i\neq j} \min\{\rank(\bP_{\xi_j,i}),R+L_{\xi_j}\}
    \nonumber\\
    & \geq 2(R+L_{\xi_j})+2 \,,
\end{align}
for all $j\in[3]$. Combining this result with Theorem~\ref{thm:partial_uniq_cp} means that the mode-$j$ factor matrix of $\tensor{Y}_{\xi_j}$ is essentially unique. This implies that condition A2 of Theorem~\ref{thm:uniqueness_deterministic} is satisfied with probability one.

\medskip
\textbf{Step 3 (auxiliary results to solve the ambiguities).} 
Using hypothesis A6, let us choose a $\xi_j\in[K]$ satisfying $\xi_j\neq\eta$. Since $\rank(\bP_{\eta,j})\geq 2$ due to assumption A4, we can find an $\bE\in\amsmathbb{R}^{\widetilde{T}\times N_{\eta,j}}$, $\widetilde{T}\geq 2$ such that $\bE\bP_{\eta_j,j}$ has full row rank. Moreover, since $\bP_{\xi_j,j}$ has full column rank due to assumption A5, $\rank(\bP_{\xi_j,j}^{\dagger})=M_j$, and due to the Sylvester rank inequality, we have
\begin{align}
    \rank(\bE\bP_{\eta,j}\bP_{\xi_j,j}^\dagger) & \geq \rank(\bE\bP_{\eta,j}) + \rank(\bP_{\xi_j,j}^{\dagger}) - M_j
    \nonumber \\
    & = \rank(\bE\bP_{\eta,j}) \,,
\end{align}
which implies $\bE\bP_{\eta,j}\bP_{\xi_j,j}^\dagger$ also has full row rank. Thus, due to the joint absolute continuity of $\bC_{j},\bD_{\xi_j,j},\bD_{\eta,j}$ (a consequence of $\eta\neq\xi_j$), using Lemma~\ref{lem:kruskalCat} we have that with probability one the Kruskal rank of $\big[\bP_{\eta,j}\bC_{j},\,\bP_{\eta,j}\bP_{\xi_j,j}^\dagger\bD_{\xi_j,j},\,\bD_{\eta,j}\big]$ is bigger than one. This holds for any $\xi_j\neq\eta$. Thus, condition A3 of Theorem~\ref{thm:uniqueness_deterministic} is satisfied with probability one.
Since A1, A2 and A3 are satisfied, \cred{Theorem~\ref{thm:uniqueness_deterministic} implies that tensors $\tensor{C}$ and $\{\tensor{D}_k\}_k$ can be uniquely recovered from $\{\tensor{Y}_k\}_k$} with probability one.
\end{proof}

The generic uniqueness result in Theorem~\ref{thm:uniqueness_generic} allows us to more clearly illustrate how considering conditions based on uni-mode uniqueness can lead to more relaxed conditions for the uniqueness of the full CTD. This can be illustrated by the following example.

\begin{example}
\label{ex:theoremConditionsRank}
Let us consider a common tensor $\tensor{C}$, with size ${7\times 11\times 9}$, and distinct components $\tensor{D}_k$, with CP ranks of $R=5$ and $L_k=5$ for all $k$, whose factor matrices are drawn from some joint continuous distribution. Suppose we have $K=3$ measured tensors $\tensor{Y}_1$, $\tensor{Y}_2$, $\tensor{Y}_3$, with sizes $10\times5\times7$, $5\times12\times7$ and $5\times7\times10$, respectively. The measured tensors satisfy the model~\eqref{eq:meas_model}, \eqref{eq:degrad_model} for some full rank matrices $\bP_{k,j}$ of appropriate dimensions.
This CTD is generically unique according to Theorem~\ref{thm:uniqueness_generic}. However, only tensor $\tensor{Y}_2$ is guaranteed to be fully unique, satisfying condition~\eqref{eq:condition_hyp_A1}. Tensors $\tensor{Y}_1$ and $\tensor{Y}_3$ are only uni-mode unique, satisfying~\eqref{eq:condition_hyp_A2} for modes $j=1$ and $j=3$, respectively.
While our result based on the uni-mode uniqueness allows us to consider ranks $R+L_k\leq10$, if we were to require $\tensor{Y}_1$ and $\tensor{Y}_3$ to be fully unique instead (satisfying~\eqref{eq:condition_hyp_A1}), uniqueness would only hold for ranks satisfying $R+L_1\leq8$ and $R+L_3\leq9$.

\end{example}

\cred{The works~\cite{sorensen2020factorizationSharesUnshared,kanatsoulis2018hyperspectralSRR_coupledCPD} consider the uniqueness of models related to the one in Section~\ref{sec:prob_statement}. However, these results are not directly comparable to ours. Flexible coupled decompositions with shared and distinct components were considered in~\cite{sorensen2020factorizationSharesUnshared}, but without a general measurement model like~\eqref{eq:meas_model},~\eqref{eq:degrad_model}. Moreover, the uniqueness results in~\cite{sorensen2020factorizationSharesUnshared} are based on algebraic and topological properties of matrices and a graph constructed from the measured tensors and factor matrices, whereas the conditions in Theorem~\ref{thm:uniqueness_generic} are directly based on the dimensions of the tensor and on the ranks of $\bP_{k,j}$. The results in~\cite{kanatsoulis2018hyperspectralSRR_coupledCPD}, which considers the spectral image fusion problem of Example~\ref{ex:HS_MS_fusion}, are easier to compare to ours. However, although both our results and those in~\cite{kanatsoulis2018hyperspectralSRR_coupledCPD} require one of the measured tensors to be fully unique,~\cite{kanatsoulis2018hyperspectralSRR_coupledCPD} does not consider distinct components (i.e., $\tensor{D}_k=\tensor{0}$) and exploits the fact that several $\bP_{k,j}$ are identity matrices. This leads to conditions for the recovery of the ``high resolution'' factor matrices that can be slightly weaker the uni-mode uniqueness used in our results.}

\section{Computing the decomposition}
\label{sec:computing_the_decomposition}

\cred{This section presents algorithms to compute the coupled CP decomposition according to the model presented in Section~\ref{sec:prob_statement}.
Two approaches are proposed. The first, presented in Section~\ref{ssec:decomp_sol_algebraic}, is a semi-algebraic strategy, while the second, presented in Section~\ref{ssec:optim_bcd_sol} uses an optimization algorithm.}

\subsection{A semi-algebraic solution}
\label{ssec:decomp_sol_algebraic}

Algebraic algorithms can provide accurate solutions in low-noise settings and constitute good quality initializations to optimization strategies~\cite{fu2020computingTensorOptimizationReview}.
Theorem~\ref{thm:uniqueness_deterministic} provides a means of devising a semi-algebraic algorithm. It \cred{requires} the CPD of at least two and at most four of the \cred{measured tensors}, and solutions to linear algebra and assignment problems, as will be detailed in the sequence. 
\cred{The approach is referred to as semi-algebraic because except for the computation of the CPDs, which might involve the use of optimization algorithms, the rest of the procedure is based on simple algebraic operations.}
To this end, let us assume the conditions in Theorem~\ref{thm:uniqueness_deterministic} to hold.

\paragraph*{\textbf{Step 0 (measurements selection)}}
Select an $\eta\in[K]$ and $\xi_j\in[K]$, $j\in[3]$ such that $\tensor{Y}_{\eta}$ is fully unique, \cred{$\tensor{Y}_{\xi_j}$ is mode-$j$ unique,} $\bP_{\xi_j,j}$ has full column rank, and the cardinality of the set $\{\eta,\xi_1,\xi_2,\xi_3\}$ is at least two. This ensures there exists a $\xi_j\neq\eta$. 
\cred{Note that there can be several possible choices of $\eta$ and $\xi_j$ satisfying these conditions, introducing a measure of user choice. First, since we will compute CPDs of the tensors indexed by $\{\eta,\xi_1,\xi_2,\xi_3\}$, the smaller the cardinality of this set the smaller the number of CPDs to compute. Moreover, we can also prioritize choices of $\eta$ or $\xi_j$ for tensors that are easier to decompose. For example, for tensors of rank $R+L_k$, typically the larger their dimensions $N_{k,j}$ the easier it is to compute its decomposition accurately (i.e., larger tensors might have less collinearity between the columns of the factor matrices, reducing convergence issues due to ``swamping''~\cite{rajih2008enhancedLineSearchPARAFAC}). In addition, if $\tensor{Y}_k$ are contaminated by noise, we might prioritize measurements with higher signal-to-noise ratio (SNR).}

\paragraph*{\textbf{Step 1 (recovery of factors separately)}}
Compute the (unique) rank-$R+L_{\eta}$ CPD of $\tensor{Y}_{\eta}$. Denote the recovered factor matrices by $\bU_{\eta,i}$, $i\in[3]$.

\paragraph*{\textbf{Step 2 (recover a different uni-mode factor)}}
Pick $j$ such that $\xi_j\neq\eta$. Since $\tensor{Y}_{\xi_j}$ is mode-$j$ unique by assumption, compute its CPD with rank $R+L_{\xi_j}$, and denote the recovered (unique) mode-$j$ factor matrix by $\bU_{\xi_j,j}$, and the remaining factors (not necessarily unique) by $\bU_{\xi_j,i}$, $i\neq j$.

\paragraph*{\textbf{Step 3 (correcting ambiguities)}}
We find the common part of the factor by searching for an optimal assignment between the transformed columns of $\bU_{\eta,j}$ and $\bU_{\xi_j,j}$. More precisely, we first map the factors to the same column space using the transformation \[\bV_{\xi_j,j}=\bP_{\eta,j}\bP_{\xi_j,j}^{\dagger}\bU_{\xi_j,j}\,,\]
recalling that $\bP_{\xi_j,j}$ is left-invertible by assumption. Now we find an optimal correspondence between $R$ columns of matrices $\bV_{\xi_j,j}$ and $\bU_{\eta,j}$. To do this, we first construct a matrix $\bZ\in\amsmathbb{R}^{R+L_{\eta} \times R+L_{\xi_j}}$ with the normalized inner product between each pairs of columns form these matrices, that is,
\begin{align}
    Z_{n,m} = \frac{|\langle \bu_n, \bv_n \rangle|}{\|\bu_n\|_2\|\bv_m\|_2} \,,
\end{align}
where $\bu_n$ is the $n$-th column of $\bU_{\eta,j}$, and $\bv_m$ is the $m$-th column of $\bV_{\xi_j,j}$. 
Then, we find a partial permutation that best aligns the common columns among these matrices, as the binary matrix $\bDelta\in\{0,1\}^{R+L_{\eta} \times R+L_{\xi_j}}$ that solves:
\begin{subequations}
\begin{align}
    \max_{\Delta_{r,s}\in \{0,1\}} & \,\, \sum_{r,s} \Delta_{r,s} \, Z_{r,s}
    \\
    {\rm s.t.} \,\,\, & \cb{1}^\top \bDelta \cb{1} = R \,, \quad \bDelta \cb{1} \leq \cb{1} \,, \quad  \bDelta^\top \cb{1} \leq \cb{1} \,.
\end{align}
\label{eq:matching_problem}
\end{subequations}
Note that this is an unbalanced assignment problem with fixed cardinality $R$, which can be solved efficiently~\cite{ramshaw2012unbalancedAssignments}.

From the solution $\widehat{\bDelta}$ to problem~\eqref{eq:matching_problem} we define two matrices: $\bM_{\eta}$, equal to the identity matrix with the columns corresponding to the indices of the all-zero rows of $\widehat{\bDelta}$ removed  (size $R\times R+L_{\eta}$), and $\bM_{\xi_j}$, which is the matrix $\widehat{\bDelta}$ with the all-zero rows removed (of size $R\times R+L_{\xi_j}$).
Now, we can select the common parts and correct the permutations for all factors of tensors $\eta$ and $\xi_j$:
\begin{align}
    \mywidehat{\bX}_{\eta,i}^C &= \bU_{\eta,i}\bM_{\eta} \,, 
    \label{eq:algeb_alg_xc_1}
    \\
    \mywidehat{\bX}_{\xi_j,i}^C &= \bU_{\xi_j,i}\bM_{\xi_j}^\top \,.
    \label{eq:algeb_alg_xc_2}
\end{align}
for $i\in\{1,2,3\}$. Finally, we compute $\mywidehat{\bC}_j$ by compensating the scaling ambiguities as $\mywidehat{\bC}_j=\bP_{\xi_j,j}^{\dagger}\mywidehat{\bX}_{\xi_j,j}^C\bLambda_j$, where $\bLambda_j$ is the diagonal matrix solving 
\begin{align}
    \bLambda_j = \mathop{\arg\min}_{\text{diagonal matrix } \bLambda} \,\,\, \big\|\mywidehat{\bX}_{\eta,j}^C - \bP_{\eta,j}\bP_{\xi_j,j}^{\dagger}\mywidehat{\bX}_{\xi_j,j}^C\bLambda\big\|_F \,.
    \label{eq:scaling_compensatot}
\end{align}
This gives us the $j$-th common factor with permutation scaling ambiguities matching those of tensor $\eta$.
\cred{The solution to~\eqref{eq:scaling_compensatot} can be computed in a straightforward manner for each diagonal element of $\bLambda_j$. Moreover, in a setting where $\mywidehat{\bX}_{\eta,j}^C$ and 
$\mywidehat{\bX}_{\xi_j,j}^C$ are not computed exactly, the solution to~\eqref{eq:scaling_compensatot} will result in an approximate scaling that best matches the two factor matrices.}

\begin{algorithm} [!t]
\footnotesize
\setstretch{1.1}
\caption{Semi-algebraic algorithm}\label{alg:alg_cp_algebraic}
\textbf{Inputs:} Data $\tensor{Y}_k$, operators $\bP_{k,j}$, ranks $R$, $L_{k}$, indices of unique tensor $\eta$ and mode-$i$ unique tensors $\xi_i$, $i\in[3]$.

Compute the CPD of $\tensor{Y}_{\eta}$, with factors $\bU_{\eta,i}$, for $i\in[3]$.

Take $j\in[3]$ such that $\xi_j\neq\eta$, compute the mode-$j$ CPD factor matrix of $\tensor{Y}_{\xi_j}$, denoted by $\bU_{\xi_j,j}$.

Compute the optimal partial assignment between the columns of $\bU_{\eta}$ and $\bU_{\xi_j}$ according to \textbf{Step 3}, solving \eqref{eq:matching_problem}--\eqref{eq:scaling_compensatot}, and compute the factor matrix
$\mywidehat{\bC}_j=\bP_{\xi_j,j}^{\dagger}\mywidehat{\bX}_{\xi_j,j}^C\bLambda_j$.

\For{$\ell\in\{1,2,3\}\setminus\{j\}$}{

\uIf{$\xi_{\ell}\in\{\eta,\xi_j\}$}{
$\mywidehat{\bC}_{\ell}=\bP_{\xi_{\ell},\ell}^{\dagger}\mywidehat{\bX}_{\xi_{\ell},\ell}^C\bLambda_{\ell}$.
}

\Else{
Compute rank $R+L_{\xi_{\ell}}$ CPD of $\tensor{Y}_{\eta_{\ell}}$ and denote the recovered mode-$\ell$ factor matrix by $\bU_{\xi_{\ell},\ell}$.

Compute the optimal partial assignment between the columns of $\mywidehat{\bC}_j$ and $\bP_{\xi_{\ell},\ell}^{\dagger}\mywidehat{\bX}_{\xi_{\ell},\ell}^C$ and the corresponding matrices as in \texttt{line 4}, and compute the factor $\mywidehat{\bC}_{\ell}$ according to \textbf{Step 4}.

}}

\textbf{return} $\tensor{C}_{\rm alg}=\big\ldbrack\mywidehat{\bC}_{1},\mywidehat{\bC}_{2},\mywidehat{\bC}_{3}\big\rdbrack$, $\tensor{D}_k=\tensor{Y}_k-\tensor{C}_{\rm alg}$, $k\in[K]$.
\end{algorithm}

\paragraph*{\textbf{Step 4 (computing the remaining factors)}}
Let us now consider the other modes, $\ell\neq j$. Let us consider the measurements available for which the remaining modes are uni-mode unique and observed with full resolution, $\{\xi_{\ell},\ell\neq j\}$. Depending on the value of $\xi_j$ and $\eta$, we have two cases:

\begin{enumerate}

\item If $\xi_{\ell}\in\{\eta,\xi_j\}$, then $\mywidehat{\bX}_{\xi_{\ell},\ell}^C$ was computed in~\eqref{eq:algeb_alg_xc_1} or~\eqref{eq:algeb_alg_xc_2} and the mode-$\ell$ factor can be computed as $\mywidehat{\bC}_{\ell}=\bP_{\xi_{\ell},\ell}^{\dagger}\mywidehat{\bX}_{\xi_{\ell},\ell}^C\bLambda_{\ell}$, with $\bLambda_{\ell}$ computed similarly to~\eqref{eq:scaling_compensatot}.

\vspace{0.3ex}

\item If $\xi_{\ell}\notin\{\eta,\xi_j\}$, then let us compute the rank $R+L_{\xi_{\ell}}$ CPD of $\tensor{Y}_{\xi_{\ell}}$, and denote the recovered mode-$\ell$ factor matrix by $\bU_{\xi_{\ell},\ell}$. Then, we repeat the same procedure as in the \textbf{Step 3}, to find the optimal assignment between the columns of matrices $\mywidehat{\bC}_j$ and $\bP_{\xi_{\ell},\ell}^{\dagger}\mywidehat{\bX}_{\xi_{\ell},\ell}^C$, which will give us the corresponding matrices $\bM_{\xi_{\ell}}$ and $\bLambda_{\ell}$. This way, we can recover $\mywidehat{\bC}_{\ell}$ as $\mywidehat{\bC}_{\ell}=\bP_{\xi_{\ell},\ell}^{\dagger}\bU_{\xi_{\ell},\ell}\bM_{\xi_{\ell}}\bLambda_{\ell}$, which matches the scaling ambiguities.
This gives us the $\ell$-th common factor with permutation scaling ambiguities matching those of $\mywidehat{\bC}_j$.

\end{enumerate}

\paragraph*{\textbf{Step 5 (reconstruct the tensor)}}
With the permutations and scalings compatible across the three modes, we can reconstruct the common tensor as $\mywidehat{\tensor{C}}=\ldbrack \mywidehat{\bC}_{1}, \mywidehat{\bC}_{2}, \mywidehat{\bC}_{3} \rdbrack$.
The distinct components $\mywidehat{\tensor{D}}_k$ are recovered as the CPD of $\tensor{Y}_k-\mathscr{P}_k(\widehat{\tensor{C}})$.

Note that, following the same arguments as in the proof of Theorem~\ref{thm:uniqueness_deterministic}, it can be shown that this recovers the true tensor.

\subsection{An \cred{optimization-based} solution}
\label{ssec:optim_bcd_sol}

\cred{Let us consider the case when the tensors are measured with noise. The CTD can be formulated as an optimization problem by seeking the decomposition satisfying the model~\eqref{eq:meas_model}, \eqref{eq:cp_z}, \eqref{eq:cp_psi} and \eqref{eq:cp_Pz} which is the closest to the measurements according to some criterion. However, the computational complexity can increase substantially when the number of measured tensors $K$ becomes large. To allow the computational complexity to be reduced, we will propose a more flexible formulation in which one can control which factors of the common tensor are coupled to a given measurement in the optimization problem. This is related to factorization frameworks with both shared and unshared factors~\cite{sorber2015structuredDataFusionReview,sorensen2020factorizationSharesUnshared}.}
To define this more precisely, let us introduce some additional notation. 
First, we can write the polyadic decomposition of tensor $\tensor{Y}_k$ as:
\begin{align}
    \tensor{Y}_k = \sum_{\cdidx\in\{C,D\}} \big\ldbrack \bX_{k,1}^\cdidx, \bX_{k,2}^\cdidx, \bX_{k,3}^\cdidx \big\rdbrack \,,
\end{align}
where $\bX_{k,j}^C$ and $\bX_{k,j}^D$ are related to the factor matrices in~\eqref{eq:cp_psi} and~\eqref{eq:cp_Pz} through:
\begin{align}
    \bX_{k,j}^C = \bP_{k,j} \bC_{j}\,,\, \quad  \bX_{k,j}^D = \bD_{k,j} \,.
    \label{eq:cp_coupl_C_Pz_Psi}
\end{align}
Let $\Gamma_j\subseteq[K]$ for $j\in[3]$  be the sets defining which coupling constraints are enforced, i.e., $k\in\Gamma_j$ means that, for mode-$j$ factor $\bC_j$ of the common tensor, the coupling  \eqref{eq:cp_coupl_C_Pz_Psi} in the $k$-the measurement is included as a constraint in the optimization problem (leaving the possibility to relax some of the constraints).
\cred{Having more flexibility in the definition of couplings was shown to be useful in prior works such as in double coupled factorizations~\cite{gong2018jointBSS_coupledCPD,chatzichristos2022fusionEEGfMRI_doubleCoupledMatrixTensor}, where the factor matrix of a given measurement tensor can be coupled to one among multiple common factors, instead of only a single global factor.}
\cred{The flexible coupled tensor decomposition considered in this work can be} formulated as follows:
\begin{subequations} \label{eq:opt_prob_cp2_full_flexible}
\begin{align}
    & \min_{\Phi} \,\, \sum_{k=1}^K \, \Big\|\tensor{Y}_k - \sum_{\cdidx\in\{C,D\}} \big\ldbrack \bX_{k,1}^\cdidx, \bX_{k,2}^\cdidx, \bX_{k,3}^\cdidx \big\rdbrack \Big\|_F^2
    \label{eq:opt_prob_cp2_a_flexible}
    \\
    & \text{s.t. } \, \bX_{k,j}^C = \bP_{k,j}\bC_{j} \,,\,\,\, k \in\Gamma_j\,, \,\, j\in[3] \,,
    \label{eq:opt_prob_cp2_b_flexible}
\end{align}
\end{subequations}
where $\Phi=\big\{\bC_{j},\bX_{k,j}^C,\bX_{k,j}^D\big\}$. Note that the factor matrices $\bD_{k,j}$ are not included in the optimization problem since they are equal to $\bX_{k,j}^D$ due to~\eqref{eq:cp_coupl_C_Pz_Psi}, and thus redundant.
The constraint~\eqref{eq:opt_prob_cp2_b_flexible} links the mode-$j$ factor of the common tensor $\tensor{C}$ among all measured datasets indexed in $\Gamma_j$. Note that we do not require knowledge of $\bP_{k,j}$ for $k\notin\Gamma_j$.
\cred{We also remark that although the loss function in~\eqref{eq:opt_prob_cp2_a_flexible} is based on the squared loss, other loss functions can also be considered, such as the Kullback-Leibler divergence or the logistic loss, which can be more natural for count or binary tensors (see, e.g., \cite{hong2020generalizedCPDlosses,pu2022stochasticMirrorDescentTensorNonEuclidean,prevost2023nonnegativeBTDbetaDivergence}).}

\begin{remark}
    \cred{Note that following the same reasoning as for the semi-algebraic algorithm developed in Section~\ref{ssec:decomp_sol_algebraic}, supposing the conditions in Theorem~\ref{thm:uniqueness_deterministic} are satisfied, is is desired that: 1) all common factors of the unique tensor are coupled, leading to $\eta\in\bigcap_{j\in[3]}\Gamma_j$; and 2) the uni-mode unique factors are coupled, i.e., for each mode $j\in[3]$, $\xi_j\in\Gamma_j$.} Moreover, for all $k$ there must be at least one mode $j$ such that $k\in\Gamma_j$, otherwise there is no coupling linking tensor $\tensor{Y}_k$. \cred{However, the uniqueness results in Theorems~\ref{thm:uniqueness_deterministic} and~\ref{thm:uniqueness_generic} concern the model~\eqref{eq:meas_model}--\eqref{eq:cp_Pz} and would have to be extended in order to rigorously address the flexible coupling case in~\eqref{eq:opt_prob_cp2_full_flexible}.}
\end{remark}

\cred{
\begin{remark}
There is a trade-off between accuracy and computational cost in the selection of $\Gamma_j$. In case the measurements $\tensor{Y}_k$ are contaminated by noise, one might want to couple as many of the measurements as possible to aid in the estimation of the common factor $\bC_j$, since this increases the amount of available data and can improve accuracy. However, including more measurements in the coupling (i.e., increasing the cardinality of $\Gamma_j$) increases the computation cost, since faster solvers are available for the case of $|\Gamma_j|\leq2$ as will be discussed in detail in Appendix~\ref{app:solution_optim_steps}.
\end{remark}
}

\cred{To solve~\eqref{eq:opt_prob_cp2_full_flexible}, different iterative approaches can be considered, such as the structured data fusion (SDF) framework~\cite{sorber2015structuredDataFusionReview}, among others~\cite{fu2020computingTensorOptimizationReview}. In this work, we consider an ALS solution, which is a common iterative approach to compute tensor decompositions since each of the updates has an easy-to-find solution~\cite{comon2009tensor_ALS_otherTales}.}
\cred{We present the overall procedure in Algorithm~\ref{alg:alg_cp}. The detailed derivations for the solution of the different optimization subproblems are presented in Appendix~\ref{app:solution_optim_steps}.}

\begin{algorithm} [!t]
\footnotesize
\setstretch{1.2}
\caption{ALS solution to problem~\eqref{eq:opt_prob_cp2_full_flexible}}\label{alg:alg_cp}
\textbf{Inputs:} Data $\tensor{Y}_k$, operators $\bP_{k,j}$, ranks $R$, $L_{k}$, couplings $\Gamma_j$. %

Initialize the factors (e.g., randomly or using Algorithm~\ref{alg:alg_cp_algebraic}) \;

\While{not converged}{

Compute $\bC_{1}$ by solving~\eqref{eq:sol_cp_Bz1} \;
Compute $\bX_{k,1}^C$ using~\eqref{eq:opt_prob_cp_Ckj_inGamma} and~\eqref{eq:opt_prob_cp_Ck1_notinGamma} \;
Compute $\bC_{2}$ by solving~\eqref{eq:sol_cp_Bz2} \;
Compute $\bX_{k,2}^C$ using~\eqref{eq:opt_prob_cp_Ckj_inGamma} and~\eqref{eq:opt_prob_cp_Ck2_notinGamma} \;
Compute $\bC_{3}$ by solving~\eqref{eq:sol_cp_Bz3} \;
Compute $\bX_{k,3}^C$ using~\eqref{eq:opt_prob_cp_Ckj_inGamma} and~\eqref{eq:opt_prob_cp_Ck3_notinGamma} \;

Compute $\bX_{k,1}^D$, $\bX_{k,2}^D$ and $\bX_{k,3}^D$ using~\eqref{eq:sol_cp_Bk1},~\eqref{eq:sol_cp_Bk2} and~\eqref{eq:sol_cp_Bk3} \;

}
\mbox{\textbf{return} $\tensor{C}=\big\ldbrack\bC_{1},\bC_{2},\bC_{3}\big\rdbrack$, $\tensor{D}_k=\big\ldbrack\bX_{k,1}^D,\bX_{k,2}^D,\bX_{k,3}^D\big\rdbrack$, $k=[K]$.}

\end{algorithm}

\section{Numerical experiments}

We evaluate the performance of the proposed decomposition methods using synthetic and real data. As a quantitative metric, we considered the normalized root mean square error (NRMSE) between the estimated and the true common component, defined as \mbox{$\text{NRMSE}=\|\mywidehat{\tensor{C}}-\tensor{C}\|_F/\|\tensor{C}\|_F$}. The accurate recovery of $\tensor{C}$ also means an accurate reconstruction of $\tensor{D}_k\approx\tensor{Y}_k-\mathscr{P}_k(\tensor{C})$ if the SNR is not too low.
In all experiments, we computed the average $\text{NRMSE}$ results over 20 Monte Carlo runs. The CPDs required by the semi-algebraic algorithms were computed using Tensorlab~\cite{tensorlab3.0}.

\vspace{-1em}

\subsection{Synthetic data}

\subsubsection{\cred{Performance as a function of the SNR}} \label{sec:ex_synth_1_SNR}
\cred{This experiment evaluates} the performance of the proposed semi-algebraic and optimization algorithms when decomposing noisy data.

\paragraph*{\textbf{Data generation}} We generated synthetic data following model \eqref{eq:meas_model}, \eqref{eq:degrad_model}, \eqref{eq:cp_z}, and \eqref{eq:cp_psi}.
To this end, we generated common tensor $\tensor{C}$, distinct tensors $\tensor{D}_k$ and $K=3$ measurements $\tensor{Y}_1$, $\tensor{Y}_2$, $\tensor{Y}_3$, all with the same sizes and rank values as in Example~\ref{ex:theoremConditionsRank}. The factor matrices of $\tensor{C}$ and $\tensor{D}_k$ (generated according to \eqref{eq:cp_z} and \eqref{eq:cp_psi}) were randomly sampled from a standard Gaussian distribution.
To generate matrices $\bP_{i,j}$ (of appropriate dimensions), we sampled each of their elements from an uniform distribution over the interval $[0,1]$.
\cred{The measurements $\tensor{Y}_1$, $\tensor{Y}_2$ and $\tensor{Y}_3$ were generated as
\begin{align}
    \tensor{Y}_k = \mathscr{P}_k(\tensor{C}) + \tensor{D}_k + \tensor{N}_k \,,
\end{align}
where $\tensor{N}_k$ denotes white Gaussian noise, with SNR given by 
\begin{align}
    \text{SNR}=10\log_{10}\bigg(\frac{\Ex\big\{\|\mathscr{P}_k(\tensor{C}) + \tensor{D}_k\|_F^2\big\}}{\Ex\big\{\|\tensor{N}_k\|_F^2\big\}}\bigg) \,,
    \label{eq:SNR_definition_synth}
\end{align}
with $\Ex\{\cdot\}$ denoting the expectation operator. Note that from the SNR in~\eqref{eq:SNR_definition_synth} the noise variance is adjusted according to the energy of the full tensor $\mathscr{P}_k(\tensor{C})+\tensor{D}_k$, which is different from the energy of the common tensor $\tensor{C}$ evaluated in the NRMSE. As} explained in Example~\ref{ex:theoremConditionsRank} this CTD is generically unique \cred{in the noiseless case}, but only tensor $\tensor{Y}_2$ is guaranteed to be fully unique, with $\tensor{Y}_1$ and $\tensor{Y}_3$ being only uni-mode unique. Moreover, the rank of each measured tensor is larger than the dimension of two of its modes, making this a numerically challenging example.

\paragraph*{\textbf{Algorithms setup}} We compared the semi-algebraic algorithm with the optimization algorithm 
following two different initialization options, with ``\textit{init. 1}'' being the solution from the semi-algebraic method, and ``\textit{init. 2}'' being a uniform random initialization. We ran the optimization algorithms for at most 1000 ALS iterations. To mitigate the random effect of the initialization, we ran all the algorithms using 50 independent initializations (including the initialization of the CPD used in the semi-algebraic approach) and picked the solution which had the lowest reconstruction errors, $\sum_k\|\tensor{Y}_k-\mywidehat{\tensor{Y}}_k\|_F^2$, with $\mywidehat{\tensor{Y}}_k$ being the reconstructed tensor.

\paragraph*{\textbf{Results}} The NRMSE results are shown in Figure~\ref{fig:nrmse_synthetic_ex}. It can be seen that the errors are very low for high SNRs (indicating a correct recovery of the true factorization) but increase approximately linearly with the noise intensity. Moreover, although the semi-algebraic solution achieved low NRMSEs for high SNRs, its performance was consistently worse than that of the optimization-based solutions. 
\cred{The two initializations options (\textit{init. 1} and \textit{init. 2}) performed similarly for high ($\geq40$ dB) SNRs, but the random initialization (\textit{init. 2}) achieved lower NRMSEs compared to the semi-algebraic one (\textit{init. 1}) for lower SNRs, particularly for the case of 20~dB. However, since the NRMSE values for this SNR are large, they do not provide a meaningful way to compare the quality of different solutions; thus, their interpretation must be performed with proper care.}

\begin{figure}
    \centering
    \includegraphics[width=0.6\linewidth]{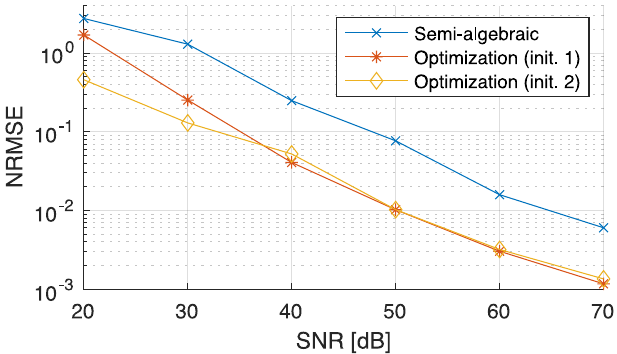}\\
    \vspace{-0.3cm}
    \caption{NRMSE for different SNRs for the example with synthetic data.}
    \label{fig:nrmse_synthetic_ex}
\end{figure}

\subsubsection{\cred{Ablation experiment 1}}

\begin{color}{red}
This experiment evaluates the performance of the proposed algorithms when there is variability in the common tensor between different datasets.

\paragraph*{\textbf{Experimental setup}} 
We consider the same general setting as in Section~\ref{sec:ex_synth_1_SNR} with an SNR fixed at 30 dB. We generate the measured tensors as
\begin{align}
    \tensor{Y}_k = \mathscr{P}_k(\tensor{C}_k(\alpha)) + \tensor{D}_k + \tensor{N}_k \,,
\end{align}
where $\tensor{C}_k(\alpha)$ represents a ``noisy'' version of the common tensor in the $k$-th dataset, computed as $\tensor{C}_k(\alpha)= (1-\alpha) \tensor{C} + \alpha \tensor{E}_k$. Tensor $\tensor{C}$, which is the same for all datasets, is generated randomly as in Section~\ref{sec:ex_synth_1_SNR}, while $\tensor{E}_k$, which is generated in the same way as $\tensor{C}$, is independent for each dataset. The contribution of each term is controlled by $\alpha\in[0,1]$. We decompose $\tensor{Y}_k$ using the proposed optimization-based algorithm with random initialization (\textit{init. 2}) for different values of $\alpha$.

\paragraph*{\textbf{Results}} The NRMSE, defined for this example as $\text{NRMSE}=\frac{1}{K}\sum_k\|\tensor{C}_k(\alpha)-\mywidehat{\tensor{C}}\|_F/\|\tensor{C}_k(\alpha)\|_F$ for the recovery of $\tensor{C}_k(\alpha)$, is shown as a function of $\alpha$ (averaged over 20 Monte Carlo runs) in Figure~\ref{fig:results_ablation_noiseInCommon}. It can be seen that the NRMSE increases with $\alpha$, which is expected since this reduces the amount of information shared among datasets. Nonetheless, the performance degradation occurs smoothly, indicating the proposed method is not overly sensitive to this type of mismodeling effect when $\alpha$ is small.
\end{color}

\begin{figure}
    \centering
    \includegraphics[width=0.45\linewidth]{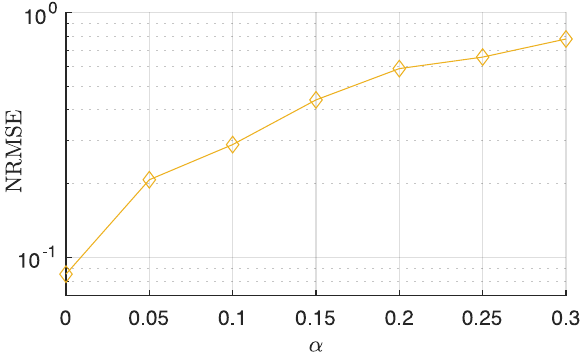}\\
    \vspace{-0.3cm}
    \caption{\cred{NRMSE (between $\tensor{C}_k(\alpha)$ and $\widehat{\tensor{C}}$) of the solutions from the optimization algorithm (\textit{init. 2}) as a function of $\alpha$. }} 
    \label{fig:results_ablation_noiseInCommon}
\end{figure}

\begin{figure}
    \centering
    \includegraphics[width=0.5\linewidth]{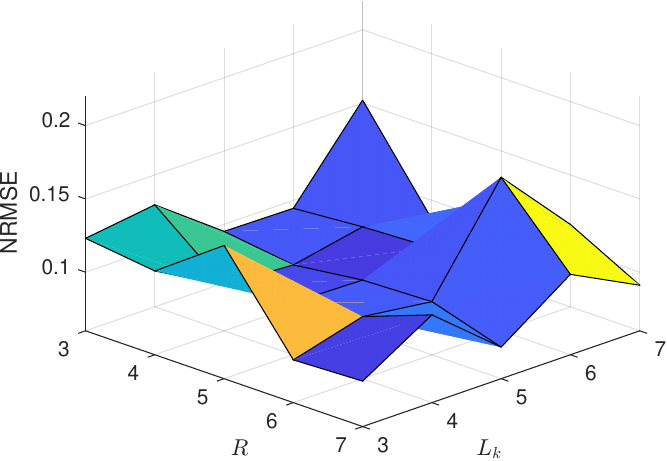}\\
    \vspace{-0.3cm}
    \caption{\cred{NRMSE of the solution given by the proposed optimization-based algorithm (\textit{init. 2}) as a function of the ranks of the decomposition $R$ and $L_k$. The true ranks of the tensors are given by $R^{\rm true}=5$ and $L_k^{\rm true}=5$.}}
    \label{fig:results_ablation_rank}
\end{figure}

\subsubsection{\cred{Ablation experiment 2}}

\begin{color}{red}
This experiment evaluates the performance of the proposed algorithm when the rank of the decomposition is misspecified.

\paragraph*{\textbf{Experimental setup}} 
We generate sets of measurements $\tensor{Y}_k$ following the same setup as in Section~\ref{sec:ex_synth_1_SNR}, fixing the SNR at 30 dB. Thus, the data are generated from a (noisy) CTD model with ranks $R^{\rm true}=5$ and $L_k^{\rm true}=5$. We then aim to recover the common tensor $\tensor{C}$ using the optimization-based algorithm with random initialization (\textit{init. 2}), but specifying different values of ranks in the range $R,L_k\in\{3,4,5,6,7\}$ to evaluate the impact of rank misspecification.

\paragraph*{\textbf{Results}} 
The NRMSE ($\|\tensor{C}-\mywidehat{\tensor{C}}\|_F/\|\tensor{C}\|_F$) as a function of $R$ and $L_k$ (averaged over 20 Monte Carlo runs) is shown in Figure~\ref{fig:results_ablation_rank}. It can be seen that as the differences between the ranks specified to the algorithm and the true ones increase, so does the NRMSE. Nonetheless, even when the specified ranks are considerably different than the true ones the NRMSE performance is still reasonable (e.g., for $R=3$, $L_k=3$ the NRMSE increases 30\% compared to the optimal value). This indicates that the proposed method is not overly sensitive to the choice of the ranks.

\end{color}

\begin{color}{red}
\subsubsection{\cred{Computation time and comparison to SDF}}

This experiment compares the execution times of the proposed algorithms, and provides comparisons to the SDF framework of~\cite{sorber2015structuredDataFusionReview}.

\paragraph*{\textbf{Experimental setup}} 
We consider the same general experimental setup as in Section~\ref{sec:ex_synth_1_SNR} while fixing the SNR at 30~dB. The SDF framework~\cite{sorber2015structuredDataFusionReview} was implemented to solve problem~\eqref{eq:opt_prob_cp2_full_flexible} using Tensorlab~\cite{tensorlab3.0}. All algorithms were implemented on Matlab\textsuperscript{TM} and run on a MacBook Pro M1 with 16GB RAM.

\paragraph*{\textbf{Results}} The NRMSE and execution times (averaged over 20 Monte Carlo runs) can be seen in Table~\ref{tab:results_time_SDF}. It can be seen that the proposed optimization-based approach (with both random and semi-algebraic initializations) and SDF obtain very similar NRMSEs, whereas the semi-algebraic approach present a higher NRMSE. In terms of computation times, the fastest approach was Algorithm~\ref{alg:alg_cp} with random initialization (\textit{init. 2}), followed by the semi-algebraic solution, whose execution is slightly slower due to an iterative optimization procedure being used to compute the CPDs. The SDF framework, although very flexible to tackle different types of problems, presented the longest execution time, significantly higher than the proposed algorithms.
\end{color}

\begin{table}[h]
\footnotesize
\centering
\begin{color}{red}
\caption{NRMSE and computation time comparison with synthetic data.}
\label{tab:results_time_SDF}
\vspace{-0.15cm}
\begin{tabular}{c|cccc}
\hline
 & Semi-algebraic & Opt. (\textit{init. 1}) & Opt. (\textit{init. 2}) & SDF~\cite{sorber2015structuredDataFusionReview} \\\hline
NRMSE & 0.9088 & 0.0770 & 0.0767 & 0.0765 \\
Time [s] & 2.776 & 6.531 & 1.026 & 23.553 \\\hline
\end{tabular}
\end{color}
\end{table}

\begin{figure}
    \centering
    \includegraphics[width=\linewidth]{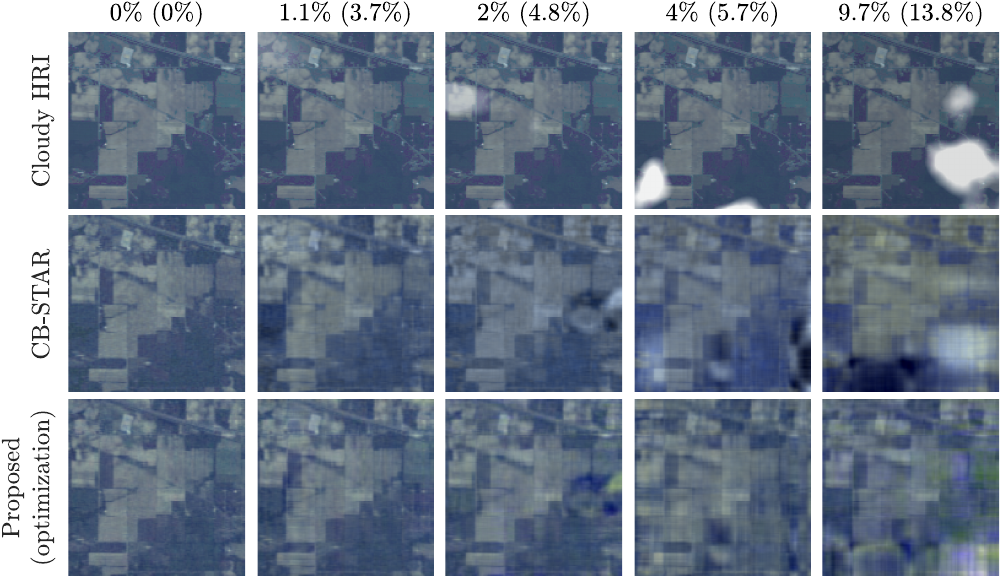}\\
    \vspace{-0.3cm}
    \caption{Illustration of common component $\tensor{C}$ and of the high resolution image reconstructed by the proposed method and by CB-STAR (the best performing competing method). The cloudy image $\tensor{X}_1$ (underlying the HSI) for the different cloud contamination levels is shown on top (the top-left corresponds to the ground truth). The percentages on top read as CC\% (CP\%).}
    \label{fig:results_exps_cloud}
    
    \vspace{-1em}
\end{figure}

\begin{table}[t]
\footnotesize
\centering
\caption{Average $\text{NRMSE}$ results for fusing an HSI and an MSI with different amounts of cloud contamination on both images.}
\vspace{-0.15cm}
\begin{tabular}{c|ccccc}
\hline
Cloud cover (CC) \%	&	0\%	&	1.1\% 	&	2\% 	&	4\% 	&	9.7\% 	\\[0.05cm]
Corrupted pixels (CP) \% 	&	0\% 	&	3.7\% 	&	4.8\% 	&	5.7\% 	&	13.8\% 	\\[0.05cm]
\hline

STEREO	&	\textbf{0.035}	&	0.082	&	0.146	&	0.222	&	0.386	\\
SCOTT	&	0.038	&	0.088	&	0.147	&	0.280	&	0.436	\\
CT-STAR	&	0.054	&	0.094	&	0.154	&	0.291	&	0.450	\\
CB-STAR	&	0.037	&	0.081	&	0.100	&	0.182	&	0.371	\\
Proposed (semi-algebraic)	&	0.079	&	0.102	&	0.125	&	0.189	&	0.297	\\
Proposed (optimization)	&	0.052	&	\textbf{0.055}	&	\textbf{0.073}	&	\textbf{0.134}	&	\textbf{0.251}	\\
\hline
\end{tabular}
\label{tab:quantitative_results}
\vspace{-1.5em}
\end{table}

\subsection{Real data}

For the experiments with real data, we consider the problem of imaging a given scene subject to cloud contamination. Two distinct acquisitions of this scene are obtained, one yielding an HSI and one an MSI, each of which is subject to some random cloud contamination (which are independent and identically distributed across the images).
This results in image-specific variability affecting both the HSI and the MSI. The objective is to recover an uncontaminated high-resolution image (HRI) of the scene by fusing the acquired HSI and MSI.

\paragraph*{\textbf{Experimental setup}} We consider an underlying high-resolution image $\tensor{C}\in\amsmathbb{R}^{145\times145\times200}$ as the Indian Pines image, with $145\times145$ pixels and $200$ spectral bands. Cloud cover maps (i.e., percentage of clouds per pixel) were then generated randomly using a nonlinear transformation of a Gaussian random field, which led to realistic cloud distribution (see the first row of Figure~\ref{fig:results_exps_cloud}) and allowed the intensity and spatial extent of the clouds to be adjusted. Specifically, we generated two cloud cover maps, $\bS_1$ and $\bS_2$ ($\in\amsmathbb{R}^{145\times145}$) to represent the cloud cover underlying the HSI and MSI, respectively. Then, cloud-corrupted images $\tensor{X}_1$ and $\tensor{X}_2$ were generated as a convex combination between the ground image and a cloud spectral signature as
\begin{align}
    \label{eq:convex_cloud_combination}
    [\tensor{X}_k]_{x,y,\lambda}=[\tensor{C}]_{x,y,\lambda}\big(1-[\bS_k]_{x,y}\big) + g_{\lambda} [\bS_k]_{x,y} \,,
\end{align}
for $k\in[2]$, where $g_{\lambda}\in\amsmathbb{R}$ is the cloud spectral reflectance at band index $\lambda$, which was computed from a real cloud signature presented in~\cite{gao1998cloudDetectionArticRegion}, and $x,y$ is the pixel indices. 
Note that decomposing $\tensor{X}_1$ and $\tensor{X}_2$ into common and distinct components is challenging since, due to the model~\eqref{eq:convex_cloud_combination},
\cred{neither the common nor the distinct components will have low CP rank. Thus, the data do not really follow the model in~\eqref{eq:cp_z},~\eqref{eq:cp_psi}, leading to an approximation problem rather than an exact decomposition.}

From the cloudy high-resolution images $\tensor{X}_1$ and $\tensor{X}_2$ we generated the HSI and the MSI, respectively, using degradation operators $\bP_{k,j}$ as described in Example~\ref{ex:HS_MS_fusion} of Section~\ref{sec:prob_statement}. We used a spatial decimation factor of four for the HSI and the spectral response of the Sentinel-2A sensor for the MSI (leading to a spectral decimation factor of 10 for the MSI). %
\cred{
The HSI $\tensor{Y}_1\in\amsmathbb{R}^{36\times36\times200}$ and the MSI $\tensor{Y}_2\in\amsmathbb{R}^{145\times145\times10}$ were then generated as
\begin{align}
    \tensor{Y}_k &= \mathscr{P}_k(\tensor{X}_k) + \tensor{N}_k
\end{align}
for $k\in\{1,2\}$, where $\mathscr{P}_k$ is as defined in~\eqref{eq:degrad_model} and $\tensor{N}_k$ are white Gaussian noise tensors whose variance is adjusted to obtain an SNR of 30~dB, which is here defined as 
\begin{align}
    \text{SNR}=10\log_{10}\bigg(\frac{\Ex\big\{\|\mathscr{P}_k(\tensor{X}_k)\|_F^2\big\}}{\Ex\big\{\|\tensor{N}_k\|_F^2\big\}}\bigg) \,.
\end{align}
For more details on the measurement model for this application, see, e.g.,~\cite{kanatsoulis2018hyperspectralSRR_coupledCPD,prevost2020coupledTucker_hyperspectralSRR_TSP,borsoi2019superResolutionHyperspectralVariabilityTIP}.
}

We adjusted the cloud cover generation setup to obtain images with a desired amount of cloud contamination, which was kept the same for the HSI and MSI and was measured using two metrics. The first metric is the cloud cover (CC) percent, defined as the mean percentage of clouds per pixel and computed as $\text{CC}=\frac{1}{Z}\sum_{k,x,y}[\bS_k]_{x,y}$, in which $Z=2\times145^2$ is a normalization equal to the total number of terms the summation. The second metric is the percent of corrupted pixels (CP), defined as the mean percentage of pixels with more than $15\%$ of cloud presence, and computed as $\text{CP}=\frac{1}{Z}\sum_{k,x,y}\iota\big([\bS_k]_{x,y}>0.15\big)$, where $\iota(\cdot)$ is an indicator function and $Z$ is defined as before. Note that the CC and CP metrics shown in the experiments are also averaged among all Monte Carlo runs.

\paragraph*{\textbf{Algorithms setup}} We compare the proposed method with the state-of-the-art CTD approaches STEREO~\cite{kanatsoulis2018hyperspectralSRR_coupledCPD}, SCOTT~\cite{prevost2020coupledTucker_hyperspectralSRR_TSP}, and with CT-STAR and CB-STAR~\cite{borsoi2020tensorHSRvariability}, the latter two which consider an image specific component in one of the images (the MSI).
The ranks of all methods were adjusted to maximize performance in a separate independently generated dataset.
We implement the optimization algorithm initialized with the algebraic approach, and considered up to 50 ALS iterations. For this example, we implemented a modification on the semi-algebraic algorithm to simplify it and improve its robustness inspired by TenRec in~\cite{kanatsoulis2018hyperspectralSRR_coupledCPD}. 
Specifically, we follow the steps 1 through 4 detailed in Section~\ref{ssec:decomp_sol_algebraic}, and exploited the fact that $\bP_{2,1}=\bP_{2,2}=\bI$ to compute \cred{$\mywidehat{\bC}_1$ and $\mywidehat{\bC}_2$} from the common parts of the factors obtained through the CPD of the MSI $\tensor{Y}_2$. However, to compute the final factor \cred{$\mywidehat{\bC}_3$} we use a different approach by just solving the linear regression problem on the HSI as \cred{$\min_{\bC_3}\|\tensor{Y}_1-\ldbrack\bP_{1,1}\mywidehat{\bC}_1,\bP_{1,2}\mywidehat{\bC}_2,\bC_3\rdbrack\|_F$.} This showed better performance than recovering $\bC_3$ through the separate CPD of the HSI $\tensor{Y}_1$.

\paragraph*{\textbf{Results}} Quantitative results are shown in Table~\ref{tab:quantitative_results} along with the CC and CP percentages. Visual results of the two overall best performing methods (the proposed optimization-based algorithm and CB-STAR) corresponding to the run with the median performance difference between them are shown in Figure~\ref{fig:results_exps_cloud}. It can be seen that the proposed methods perform slightly worse when there are no clouds, which is natural as the models in the competing methods (without personalized components) match the dataset better. However, as soon as cloud cover is present (even for 1.1\%) the proposed optimization-based approach gets considerably lower NRMSE then the other algorithms. The semi-algebraic algorithm performed considerably worse than the optimization one for low CC percentages, with their relative performances \cred{getting} better as the cloud contamination increases. The visual results also corroborate these findings, with the reconstruction by CB-STAR showing more artifacts compared to the proposed method reconstruction. The artifacts also get increasingly strong as the cloud contamination increases, being subtle for a CC of $1.1\%$ but very strong for a CC of $9.7\%$.

\section{\cred{Conclusions and perspectives}}
\label{sec:conclusions}

In this paper, a personalized CTD framework was proposed. A flexible model was presented representing the datasets as the sum of two components, each admitting a low-rank CPD. The first component was linked to a common tensor through a multilinear measurement model, while the second component was distinct and captured dataset-specific information. This model generalized several existing CPDs. The generic uniqueness of the decomposition was shown to hold under mild conditions that highlighted the influence of the measurement model and of ``weaker'' uni-mode uniqueness of individual datasets on the uniqueness of the full decomposition.
Assuming the ranks of the common and distinct components to be known, two algorithms were proposed to compute the decomposition.
The first was semi-algebraic, performing well for high SNRs but being less robust in low-SNR settings. The second was based on an optimization procedure and showed better performance, especially on noisy and real data. Experiments illustrated the advantage of the proposed framework compared with the state of the art methods.

\cred{An important issue related to the practical use of the proposed CTD framework is the estimation of its ranks, $R$ and $L_k$. While several strategies have been investigated to estimate the number of common components in statistical models involving multiple datasets (see, e.g.,~\cite{akhonda2021disjointSubspacesCommonAndDistinctComponentsfMRI,bhinge2017estimationCommonOrderMultipleDatasetsfMRI,hasija2020determiningDimensionSubspaceCorrelatedMUltipleDatasets}), these approaches are based on a very different set of model assumptions. Thus, in future work we aim to investigate the order selection problem for CTDs with common and distinct components from both theoretical and experimental perspectives.}

\appendices

\begin{color}{red}
\section{Solution to the optimization problem~\eqref{eq:opt_prob_cp2_full_flexible}}
\label{app:solution_optim_steps}

\end{color}

\paragraph*{\textbf{Optimization w.r.t. $\bC_{j}$}} 

Let us first consider the solution for $\bC_{1}$. Incorporating the equality constraint~\eqref{eq:opt_prob_cp2_b_flexible} directly in the cost function, this problem can be written as:
\begin{align}
    & \min_{\bC_{1}} \sum_{k\in\Gamma_1} \Big\| \widetilde{\tensor{Y}}_k - \big\ldbrack \bP_{k,1}\bC_{1}, \bX_{k,2}^C, \bX_{k,3}^C \big\rdbrack \Big\|_F^2
    \label{eq:opt_cp_bcd_bz_1}
\end{align}
where $\widetilde{\tensor{Y}}_{k}=\tensor{Y}_k-\big\ldbrack \bX_{k,1}^D, \bX_{k,2}^D, \bX_{k,3}^D \big\rdbrack$. Note that this is equivalent to a coupled CP decomposition without distinct components. \cred{Using the properties of the mode-1 matricization, it can be shown that} the critical points of~\eqref{eq:opt_cp_bcd_bz_1} satisfy
\begin{align}
    \sum_{k\in\Gamma_1} & \! \big(\cred{\bJ_{k,1}^\top\bJ_{k,1}}\big) \bC_{1}^\top \big(\bP_{k,1}^\top \bP_{k,1}\big)
    - \cred{\bJ_{k,1}^\top} \unfold{\widetilde{\bY}_{\!k}}{1} \bP_{k,1} = \cb{0} \,,
    \label{eq:sol_cp_Bz1}
\end{align}
where $\cred{\bJ_{k,1}}=\bX_{k,3}^C \odot \bX_{k,2}^C$. Note that the product \cred{$\bJ_{k,1}^\top\bJ_{k,1}$} can be computed using mixed product properties (see, e.g., \cite[eq.~2.2]{kolda2009tensor}). \cred{Since~\eqref{eq:sol_cp_Bz1} is linear in $\bC_1\in\amsmathbb{R}^{M_1\times R}$, it can be solved with a complexity of at most $O\big((M_1R)^3\big)$ operations. 
However, since this consists in a generalization of the Sylvester equation, when $|\Gamma_1|\leq2$ specialized solvers can be much faster: the complexity of solving the generalized Sylvester equation in this case is $O(M_1^3 + R^3)$ operations when using the Hessenberg-Schur or Bartels-Stewart methods~\cite{bartels1972solutionSylvester,gardiner1992solutionSylvester,golub1979hessenbergSchurSolverSylvesterEq,simoncini2016computationalSylvesterEqReview}.}

Optimizing the cost function w.r.t. $\bC_{2}$ and $\bC_{3}$ can be done similarly by using the mode-2 and mode-3 tensor matricizations, which leads to the following equations:
\begin{align}
    \sum_{k\in\Gamma_2} & \! \big(\cred{\bJ_{k,2}^\top\bJ_{k,2}}\big) \bC_{2}^\top \big(\bP_{k,2}^\top \bP_{k,2} \big)
    - \cred{\bJ_{k,2}^\top} \unfold{\widetilde{\bY}_{\!k}}{2} \bP_{k,2} = \cb{0} \,,
    \label{eq:sol_cp_Bz2}
    \\
    \sum_{k\in\Gamma_3} & \! \big(\cred{\bJ_{k,3}^\top\bJ_{k,3}}\big) \bC_{3}^\top \big(\bP_{k,3}^\top \bP_{k,3}\big)
    -  \cred{\bJ_{k,3}^\top} \unfold{\widetilde{\bY}_{\!k}}{3} \bP_{k,3} = \cb{0} \,,
    \label{eq:sol_cp_Bz3}
\end{align}
where $\cred{\bJ_{k,2}}=\bX_{k,3}^C \odot \bX_{k,1}^C$ and $\cred{\bJ_{k,3}}=\bX_{k,2}^C \odot \bX_{k,1}^C$.

\paragraph*{\textbf{Optimization w.r.t. $\bX_{k,j}^C$}}
The solution to this problem depends on whether $\bX_{k,j}^C$ is included in the set of equality constraints~\eqref{eq:opt_prob_cp2_b_flexible}. If $k\in\Gamma_j$, then the solution is direct:
\begin{align}
    \bX_{k,j}^C = \bP_{k,j}\bC_{j}\,, \quad k\in\Gamma_j, \,\, j\in[3] \,.
    \label{eq:opt_prob_cp_Ckj_inGamma}
\end{align}
On the other hand, if $k\not\in\Gamma_j$, the optimization problem for the mode $j=1$ will be:
\begin{align}
    & \min_{\bX_{k,1}^C} \big\| \widetilde{\tensor{Y}}_k - \big\ldbrack\bX_{k,1}^C, \bX_{k,2}^C, \bX_{k,3}^C \big\rdbrack \big\|_F^2 \,,
\end{align}
where $\widetilde{\tensor{Y}}_{k}=\tensor{Y}_k-\big\ldbrack \bX_{k,1}^D, \bX_{k,2}^D, \bX_{k,3}^D \big\rdbrack$. The solution to this problem is given by: 
\begin{align}
    \bX_{k,1}^C = \big[\big(\bX_{k,3}^C \odot \bX_{k,2}^C\big)^\dagger \unfold{\widetilde{\bY}_{\!k}}{1}\big]^\top \,, \quad k\not\in\Gamma_1 \,.
    \label{eq:opt_prob_cp_Ck1_notinGamma}
\end{align}
The same steps can be used to compute the solutions for modes $j=2$ and $j=3$:
\begin{align}
    \bX_{k,2}^C ={}& \big[\big(\bX_{k,3}^C \odot \bX_{k,1}^C\big)^\dagger \unfold{\widetilde{\bY}_{\!k}}{2}\big]^\top \,,\quad k\not\in\Gamma_2 \,,
    \label{eq:opt_prob_cp_Ck2_notinGamma}
    \\
    \bX_{k,3}^C ={}& \big[\big(\bX_{k,2}^C \odot \bX_{k,1}^C\big)^\dagger \unfold{\widetilde{\bY}_{\!k}}{3}\big]^\top \,, \quad k\not\in\Gamma_3 \,.
    \label{eq:opt_prob_cp_Ck3_notinGamma}
\end{align}

\paragraph*{\textbf{Optimization w.r.t. $\bX_{k,j}^D$}}
This problem is given by
\begin{align}
    & \min_{\bX_{k,1}^D,\bX_{k,2}^D,\bX_{k,3}^D} \,\, \big\| \overline{\tensor{Y}}_k - \big\ldbrack\bX_{k,1}^D, \bX_{k,2}^D, \bX_{k,3}^D \big\rdbrack \big\|_F^2 \,,
    \label{eq:opt_cp_bcd_bk}
\end{align}
where $\overline{\tensor{Y}}_k=\tensor{Y}_k-\big\ldbrack\bX_{k,1}^C, \bX_{k,2}^C, \bX_{k,3}^C \big\rdbrack$. Problem~\eqref{eq:opt_cp_bcd_bk} consists of computing a rank-$L_k$ CPD of $\overline{\tensor{Y}}_k$.
Following the ALS strategy, we update $\bX_{k,j}^D$ sequentially over $j\in[3]$. Using the mode-$j$ matricization and proceeding in the same way as in the previous problem, the solutions are given by:
\begin{align}
    \bX_{k,1}^D ={} & \big[\big(\bX_{k,3}^D \odot \bX_{k,2}^D\big)^{\dagger}\,\unfold{\overline{\bY}_{\!k}}{1}\big]^\top \,,
    \label{eq:sol_cp_Bk1}
    \\
    \bX_{k,2}^D ={} & \big[\big(\bX_{k,3}^D \odot \bX_{k,1}^D\big)^{\dagger}\,\unfold{\overline{\bY}_{\!k}}{2} \big]^\top \,,
    \label{eq:sol_cp_Bk2}
    \\
    \bX_{k,3}^D ={} & \big[\big(\bX_{k,2}^D \odot \bX_{k,1}^D\big)^{\dagger}\,\unfold{\overline{\bY}_{\!k}}{3} \big]^\top \,.
    \label{eq:sol_cp_Bk3}
\end{align}

\bibliographystyle{IEEEtran}
\bibliography{references}

\vspace*{-2ex}

\begin{IEEEbiography}[{\includegraphics[width=1in,height=1.25in,clip,keepaspectratio]{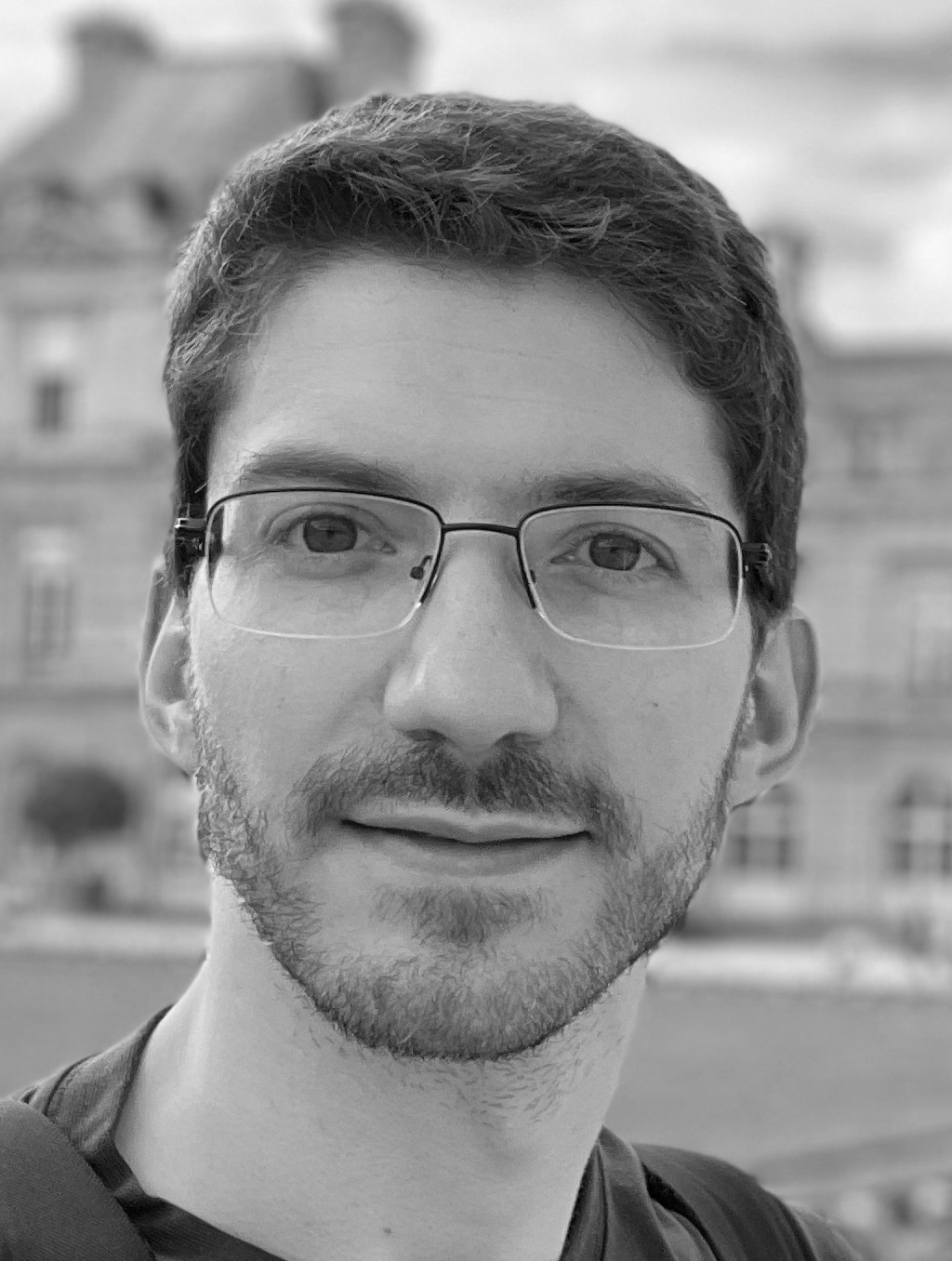}}]
{Ricardo Borsoi}(Member, IEEE) 
received the Doctorate degree from the Federal University of Santa Catarina, Brazil, and from the C\^ote d'Azur University, France, in 2021. He is currently a permanent research scientist in the French National Center for Scientific Research (CNRS) in the CRAN laboratory, Nancy, France. His research focuses on statistical inference, tensor decompositions and inverse problems, with applications in hyperspectral and medical image analysis.
\end{IEEEbiography}

\vspace*{-2ex}

\begin{IEEEbiography}[{\includegraphics[width=1in,height=1.25in,clip,keepaspectratio]{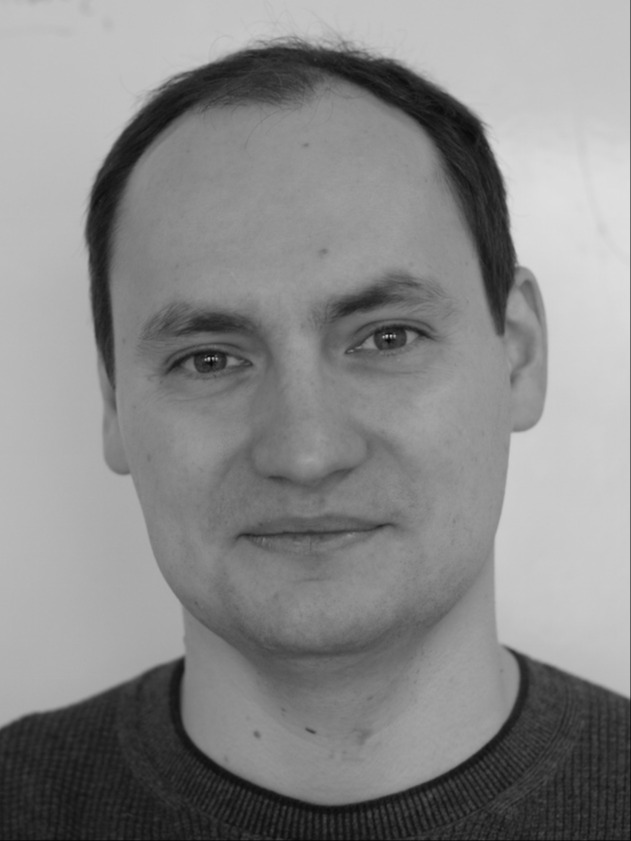}}]
{Konstantin Usevich}(Member, IEEE)
obtained a  Ph.D in 2011, from St.\,Petersburg State University, Russia. He held  several postdoctoral positions in UK, Belgium, and France, before joining the CNRS in 2017. He is currently a permanent CNRS researcher at CRAN laboratory, Nancy, France.
His research interests linear and multilinear algebra, optimization, signal and image processing, and machine learning. He has been a member of the editorial board of SIAM Journal on Matrix Analysis and Applications since 2018.
\end{IEEEbiography}

\vspace*{-2ex}

\begin{IEEEbiography}[{\includegraphics[width=1in,height=1.25in,clip,keepaspectratio]{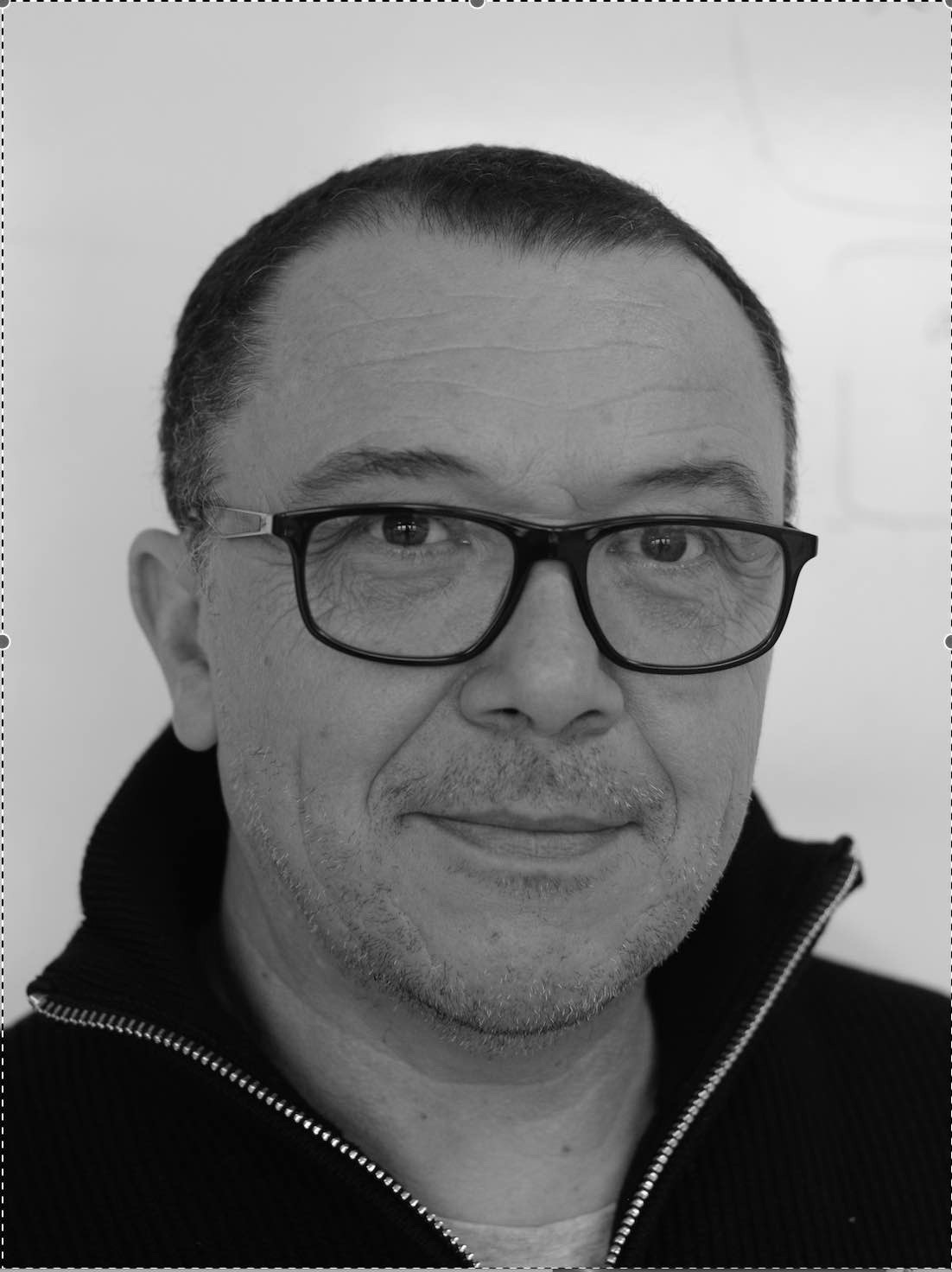}}]
{David Brie}(Member, IEEE) received the Ph.D. degree in signal processing from Université Henri Poincaré, Nancy, France, in 1992. He is currently a Full Professor with the Université de Lorraine, Nancy, France. He is a Member of the Centre de Recherche en Automatique de Nancy, Nancy, France. His research focuses on multidimensional statistical signal processing. He served as Editor in Chief of the French journal Traitement du Signal from 2014 to 2019. He  was the General Co-Chair with Jean-Yves Tourneret of IEEE CAMSAP 2019 Workshop and general chair of GRETSI 2022 conference.
\end{IEEEbiography}

\vspace*{-2ex}

\begin{IEEEbiography}[{\includegraphics[width=1in,height=1.25in,clip,keepaspectratio]{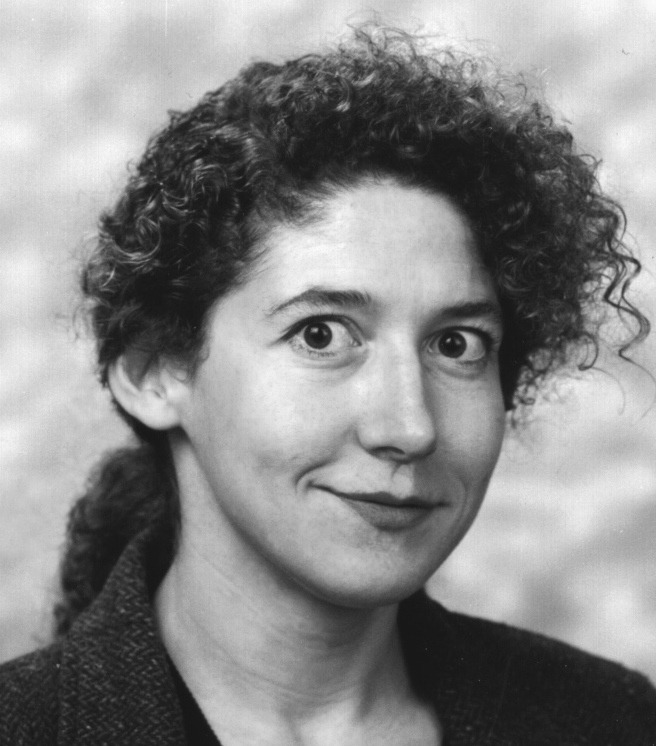}}]
{Tülay Adali}(Fellow, IEEE) is a Distinguished University Professor at the University of Maryland Baltimore County (UMBC), Baltimore, MD. Prof. Adali has been active within the IEEE. She served as the Chair of the IEEE Brain Technical Community in 2023, and the Signal Processing Society (SPS) Vice President for Technical Directions 2019-2022. She is currently the editor-in-chief of the IEEE Signal Processing Magazine. Her roles in conference organization include general/technical chair position for the IEEE Machine Learning for Signal Processing (MLSP) and Neural Networks for Signal Processing (NNSP) Workshops 2001-2009, 2014, and 2023, and program/technical chair position for the IEEE International Conference on Acoustics, Speech, and  Signal Processing (ICASSP) in 2017 and 2026 among other roles.  She was the Chair of the NNSP/MLSP Technical Committee, 2003-2005 and 2011-2013, and served on numerous boards and technical committees of the SPS. Prof. Adali is a Fellow of the IEEE, AIMBE, and AAIA, a Fulbright Scholar, an IEEE SPS Distinguished Lecturer, and the UMBC Presidential Research Professor for 2024-2027. She is the recipient of SPS Meritorious Service Award, Humboldt Research Award, IEEE SPS Best Paper Award, SPIE Unsupervised Learning and ICA Pioneer Award, the University System of Maryland Regents' Award for Research, and the NSF CAREER Award. Her current research interests are in statistical signal processing and machine learning, with applications to neuroimaging data analysis. 
\end{IEEEbiography}

\end{document}